\PassOptionsToPackage{table}{xcolor}
\documentclass[10pt,letterpaper]{article}
\usepackage[T1]{fontenc}
\usepackage{amsmath,amssymb}
\usepackage{amsthm}
\usepackage{trajectory2gym_preprint}
\usepackage{algorithm}
\usepackage{algpseudocode}
\theoremstyle{plain}\newtheorem{proposition}{Proposition}
\theoremstyle{remark}\newtheorem{remark}{Remark}
\usepackage{graphicx}
\usepackage{booktabs,multirow,array,tabularx}
\usepackage{float}
\usepackage{wrapfig}
\usepackage{needspace}
\usepackage{placeins}
\usepackage{xcolor}
\newcolumntype{P}[1]{>{\raggedright\arraybackslash}p{#1}}
\newcolumntype{Y}{>{\raggedright\arraybackslash}X}
\usepackage{enumitem}
\usepackage{hyperref}
\usepackage{url}
\usepackage{fontawesome5}
\hypersetup{
  colorlinks=true,
  linkcolor=black,
  citecolor=black,
  urlcolor=black,
  pdftitle={Trajectory2Gym: Decision-Grounded Compositional Interaction Synthesis for Tool-Use Agents},
  pdfauthor={Yihao Hu, Zhihao Wen, Yitong Lin, Siyuan Lu, Pan Wang, Wei Wu, Jiang-Ming Yang, Tao Lin}
}

\definecolor{T2GTableBlue}{HTML}{E1F2FF}
\definecolor{GroupGray}{HTML}{EEEEEE}
\definecolor{SurveyGreen}{HTML}{5F9F68}
\definecolor{SurveyOrange}{HTML}{CC7C48}
\newcommand{\tabsection}[2]{\multicolumn{#1}{@{}l}{\itshape #2}\\}
\newcommand{\TgMainTableStyle}{%
  \footnotesize\normalfont\rmfamily
  \renewcommand{\arraystretch}{1.0}}
\newcolumntype{R}{>{\raggedleft\arraybackslash}p{3.5em}}
\newcolumntype{N}{>{\raggedleft\arraybackslash}X}

\newcommand{\stitle}[1]{\par\noindent\textbf{#1}\enspace}

\newcommand{\chone}{Challenge~1}
\newcommand{\chtwo}{Challenge~2}

\newcommand{\env}{\mathcal{E}}

\newcommand{\method}{Trajectory2Gym}
\title{Trajectory2Gym: Decision-Grounded Compositional Interaction Synthesis for Tool-Use Agents}

\author{Yihao~Hu\textsuperscript{1,2,*} \and
Zhihao~Wen\textsuperscript{2,*} \and
Yitong~Lin\textsuperscript{2} \and
Siyuan~Lu\textsuperscript{1}\\[2pt]
Pan~Wang\textsuperscript{2} \and
Wei~Wu\textsuperscript{2,\ensuremath{\dagger}} \and
Jiang-Ming~Yang\textsuperscript{2} \and
Tao~Lin\textsuperscript{1,\ensuremath{\dagger}}}
\renewcommand{\TgAffiliations}{%
  \textsuperscript{1}Westlake University (\href{https://westlake.edu}{westlake.edu})
  \quad
  \textsuperscript{2}Ant International (\href{https://ant-intl.com}{ant-intl.com})}
\renewcommand{\TgAuthorNote}{%
  \textsuperscript{*}Equal contribution.\quad
  \textsuperscript{\ensuremath{\dagger}}Corresponding authors.}
\renewcommand{\TgLinks}{%
  \href{https://yihaohu0118.github.io/Trajectory2Gym-page/}{\faGlobe\hspace{0.35em}Homepage}%
  \hspace{1.25em}%
  \href{https://github.com/yihaohu0118/Trajectory2Gym}{\faGithub\hspace{0.35em}Code}%
  \hspace{1.25em}%
  \href{https://huggingface.co/datasets/mis2/Trajectory2gym}{\faDatabase\hspace{0.35em}Hugging Face}}
\date{October 24, 2026}

\begin{document}

\maketitle

\begin{abstract}
Successful tool-use trajectories are not necessarily useful supervision.
Even successful interactions may underrepresent multi-step dependencies or
contain actions that are unjustified by the information visible to the agent
at decision time.
We identify two coupled requirements for effective trajectory supervision:
compositional coverage, so that important observation--action dependencies
are represented, and decision grounding, so that each supervised action is
supported by the learner-visible interaction history.
We introduce Trajectory2Gym, which constructs executable interactions around
under-supported compositional structures and derives supervision only from
grounded decisions in their realized executions.
Structural-support selection allocates the resulting experience, while
parent-normalized supervised fine-tuning preserves the semantic interaction
as the unit of training weight.
This reframes trajectory generation as the deliberate construction of
learnable interaction experience, rather than the collection of successful
trajectories.
Across five Qwen3 scales, Trajectory2Gym improves Tau1, Tau2, and BFCL at
every tested scale, increasing the equal-weight three-benchmark mean by
16.9--29.8 percentage points.
The gains are strongest in stateful and multi-turn interaction, where later
decisions depend on earlier observations and feedback.
Our results suggest that scaling tool-use agents depends not only on
collecting more successful trajectories, but on constructing interactions
whose dependencies are represented and whose decisions are learnable from
the agent's visible history.
\end{abstract}

\section{Introduction}
\label{sec:introduction}

Ask a support agent to cancel two orders and update a third. A competent
agent looks up the account; when the email lookup fails, it asks for a
name and ZIP code and retries; it reads each order to check it is
pending, confirms with the user, and only then cancels and swaps items
using catalog-returned identifiers. Each call is simple; what must be
learned is their composition. We call such workflows \emph{compositional
interactions}: an observation, user update, or state change determines a
later decision. Tool-use agents face this composition constantly
\citep{yao2023react,schick2023toolformer,patil2023gorilla,
qin2024toollearningfoundationmodels,qu2024toollearninglargelanguage},
interactive benchmarks make it central to task success
\citep{yao2024taubench,barres2025tau2bench,patil2025bfcl}, and collected
or synthesized trajectories are its standard supervision
\citep{qin2023toolllm,zeng2023agenttuning}.

Recent pipelines show that \emph{what} interactions are collected matters,
yet the prevailing strategy still filters trajectories by task
completion or per-call execution
\citep{liu2024apigen,liu2025toolace,prabhakar2025apigenmt,xu2025toucan},
and fine-tuning directly on logged trajectories can even fall below the
initialization \citep{wu2026terminaluniverseturningagenttrajectories}.
Success certifies \emph{that} an episode ended well, not
\emph{what} it teaches or \emph{whether} each imitated action was
justified by what the agent could see. Two challenges follow.
\textbf{Challenge 1 (compositional coverage):} trajectories must
exercise cross-step dependencies with repeated support, yet in our pool
of 34,772 certified candidates, 40.9\% contain at most two supervised
decisions and only 10.0\% contain a certified decision after an
authentic error. \textbf{Challenge 2 (decision grounding):} each target
must follow from decision-time visible information; successful API
execution alone cannot establish this. A collector may use an unseen
identifier, and in Figure~\ref{fig:trajectory2gym-teaser} the API
accepts a stale two-item exchange as readily as the lamp-only exchange
finally requested by the user.

\begin{figure*}[t]
  \centering
  \includegraphics[width=0.98\textwidth]{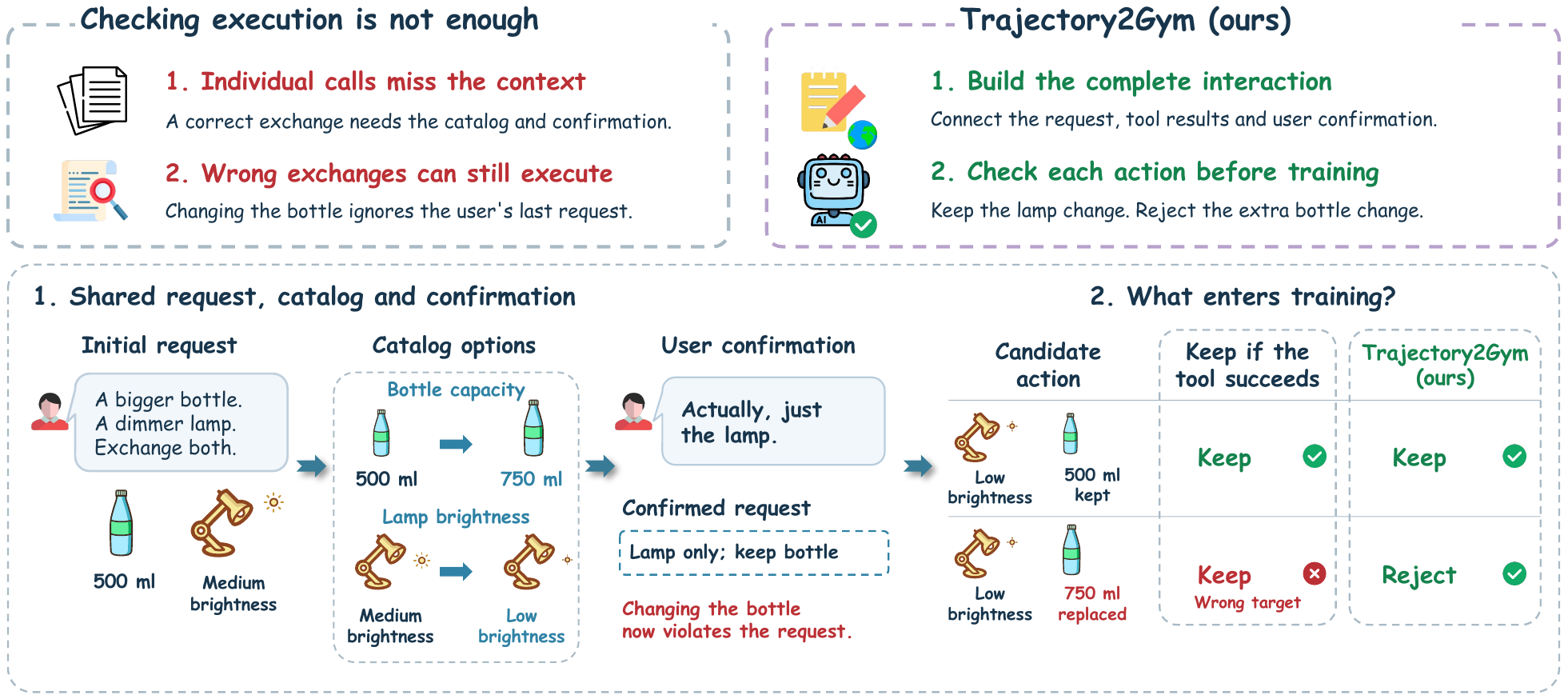}
  \caption{\textbf{Compositional interaction requires grounded decisions.}
Top: \method{} constructs linked interactions (1) and checks each training
action against that history (2). Bottom: a Tau2 Retail user requests a
larger bottle and a dimmer lamp, then confirms only the lamp. Both rules
judge the same two candidate actions. Filtering by execution keeps both
because native tool replay accepts both requests. \method{} uses catalog
evidence and the latest confirmation to keep the lamp exchange and reject
the additional bottle exchange. Green checks mark correct selection
decisions; the red cross marks erroneous retention. This schematic illustrates supervision
selection with paraphrased dialogue and a constructed second candidate.}
  \label{fig:trajectory2gym-teaser}
\end{figure*}

Our central claim is that
\textbf{what a trajectory teaches is decided not by its outcome, but by
the dependencies it realizes and the decisions its visible history can
justify}.
\method{} operationalizes this claim in three stages
(Figure~\ref{fig:trajectory2gym-overall}).
(1)~\emph{Compositional interaction synthesis} builds and executes tasks
around under-supported interaction paths in native environments, so that
information acquisition and recovery are demonstrated (\chone).
(2)~\emph{Grounded supervision} supervises only decisions certified by
visible evidence, execution, and workflow order, keeping authentic
errors as recovery context (\chtwo).
(3)~\emph{Structural-support learning} selects trajectories to cover
capability paths at a per-scale support depth and trains with
parent-normalized supervised fine-tuning (SFT). The stages are coupled:
synthesis without grounding would teach the collector's shortcuts,
grounding without synthesis would certify a pool where recovery and
discovery stay rare, and selection without both would only re-weight
that pool. 
With plain supervised fine-tuning from each public initialization,
the paired gains measure the complete post-training recipe.

Data built this way lifts every scale: \method{} improves Tau1, Tau2,
and BFCL on all five Qwen3 models from 1.7B to 30B, raising the
three-benchmark mean by 16.9 to 29.8 points. The lift outweighs model
size: the post-trained 4B model surpasses the public 30B-A3B
initialization on Tau2 (69.17 versus 28.33), and, against public
agent-training checkpoints used as references rather than matched
controls, the 8B model has the highest three-benchmark mean. 
Our contributions are:
\begin{itemize}[topsep=-1pt,itemsep=1.5pt,parsep=0pt,partopsep=0pt,leftmargin=*]
  \item We identify compositional coverage and decision grounding as
  coupled requirements and formalize them as a decision-level supervision
  mask over agent-visible histories (Section~\ref{sec:formulation}).
  \item We introduce \method{}, which synthesizes under-supported
  interactions in native environments and certifies targets from their
  executed, agent-visible histories (Sections~\ref{sec:synthesis}
  and~\ref{sec:certification}).
  \item We develop structural-support selection with parent-normalized
  learning and show that the framework improves all three benchmarks at
  each of five scales, with controlled comparisons of every stage
  (Sections~\ref{sec:training} and~\ref{sec:experiments}).
\end{itemize}

\section{Related Work}
\label{sec:related}

\stitle{Tool-use benchmarks and training.}
Single-turn benchmarks test tool and argument selection
\citep{li2023apibank,zhuang2023toolqadatasetllmquestion,
huang2024metatoolbenchmarklargelanguage,chen-etal-2024-eval};
later suites add nested, multi-hop, and long-context calls
\citep{basu-etal-2025-nestful,ye-etal-2025-toolhop,
zhong2025complexfuncbenchexploringmultistepconstrained,
chen2025acebenchwinsmatchpoint}.
Interactive benchmarks score whole episodes: multi-turn tool use
\citep{wang2024mintevaluatingllmsmultiturn,liu2023agentbench,
lu2024toolsandbox,guo2024stabletoolbench}, $\tau$-bench customer service
\citep{shi2026tauknowledge}, web, computer, and workplace tasks
\citep{zhou2023webarena,xie2024osworld,mialon2024gaia,trivedi2024appworld,
xu2025theagentcompanybenchmarkingllmagents,
li2026tooldecathlonbenchmarkinglanguage}, and MCP ecosystems
\citep{luo2025mcpuniverse,mo2025livemcpbench,
wang2025mcpbenchbenchmarkingtoolusingllm,
wu2025mcpmarkbenchmarkstresstestingrealistic}.
Our focus is on representing the cross-step dependencies that an
executed trajectory actually exercises.
Training uses synthesized demonstrations
\citep{parisi2022talm,tang2023toolalpaca,
yang2023gpt4toolsteachinglargelanguage},
cross-environment action traces
\citep{chen2023fireact,yin-etal-2024-agent,
chen-etal-2024-agent,zhang-etal-2025-xlam,song-etal-2024-agentbank,
xi2024agentgym}, function-calling specialization
\citep{lin2024hammerrobustfunctioncallingondevice,
wang2025toolgenunifiedtoolretrieval}, and exploration or reinforcement
learning with execution rewards
\citep{song2024eto,putta2024agentq,qian2025toolrl,
feng2025retoolreinforcementlearningstrategic,
li2025torlscalingtoolintegratedrl,
zhang2025nemotronresearchtooln1exploringtoolusinglanguage,
xi2025agentgymrl,zhang2026fissiongrpo,
kimiteam2026kimik2openagentic}.
SEAL couples diagnosis-guided training-interface adaptation with policy
updates \citep{hu2026seal}, while FACA uses next-turn user reactions for
local credit assignment \citep{zhao2026towards}.
Outcome-filtered training commonly treats successful trajectories as
imitable units, while failure-aware methods add step-level signals
\citep{wang2024learningfailureintegratingnegative,xiong-etal-2024-watch}.
DARC uses failure diagnosis to select recovery interventions
\citep{wang2026diagnosis}.
Our focus is target eligibility under the agent-visible history.

\stitle{Trajectory synthesis and curation.}
Trajectory construction is a data-design problem: verified executable
synthesis produces function-calling data at scale
\citep{yang2025toolmind}, multi-turn
extensions add dialogues, tool graphs, and simulated users
\citep{chen2025facilitatingmultiturnfunctioncalling,yin-etal-2025-magnet,
wang-etal-2025-toolflow}, and web and GUI agents mine tutorials,
indirect knowledge, reverse task synthesis, and exploration
\citep{xu2025agenttrekagenttrajectorysynthesis,
ou2024synatraturningindirectknowledge,sun-etal-2025-os,
pahuja-etal-2025-explorer}.
Capability-oriented pipelines target changing or infeasible intents
\citep{wang2026trajectory2task}, active visual acquisition
\citep{zhang2026visualseeker}, or compositional tool graphs
\citep{lu2023chameleonplugandplaycompositionalreasoning,
shen2024taskbenchbenchmarkinglargelanguage}; others scale environments
\citep{hu2025agentgenenhancingplanningabilities,song2025envscaler,
tu2025scaleenv,fang2025agentscaler,chen2026dreamgym,
wang2026agentworldmodel,tian2026astra},
rebuild executable workspaces from logged trajectories and re-solve them
under hidden tests,
or standardize formats \citep{song2026agentdataprotocolunifying}.
Code2Skill derives and verifies reusable procedural skills from source
code \citep{tong2026grounded}.
Executable tool-use synthesis often verifies outcomes or calls.
For software-engineering agents, Patches-to-Trajectories curates
prefix-grounded, efficient steps using privileged reference patches
\citep{ma2026patchestotrajectories}.
Curation is complementary: data selection favors quality over volume
\citep{zhou2023limaalignment,liu2024makesgooddataalignment}, TopoCurate
selects by interaction topology \citep{yang2026topocurate}, and imitation
learning separates privileged demonstrator knowledge from student
observations \citep{ross2011dagger,chen2019learningcheating}.
\method{} combines deficit-guided synthesis in native tool-use
environments, policy-visible target eligibility, and structural-support
allocation across learner scales.

\section{Problem Formulation}
\label{sec:formulation}

Let $\env$ be a tool-use environment and $V$ the projection of an
interaction history onto what the agent can see: we take the perspective
of a developer who records the complete log during collection but trains
a student that acts only through the native interface. Given a request
$x$ and tool schema $\mathcal T_t$, the policy acts from
\begin{equation}
    c_t
    =
    \bigl(x,\mathcal T_t,V(H_{<t})\bigr),
    \qquad
    a_t\sim\pi_\theta(\cdot\mid c_t),
    \label{eq:tool-use-policy}
\end{equation}
where $H_{<t}$ is the history preceding decision $t$; the environment
executes $a_t$, returns an observation, and updates its state, and a
completed trajectory $\tau$ receives a terminal outcome
$R_{\env}(\tau)\in\{0,1\}$ from the native checker.

For learner scale $s$ we construct a supervised trajectory set
\begin{equation}
    \mathcal D_s
    =
    \{(\tau_i,\mathbf m_i)\}_{i=1}^{N_s},
    \qquad
    \mathbf m_i\in\{0,1\}^{T_i},
    \label{eq:trajectory-supervision}
\end{equation}
where $T_i$ is the number of assistant decisions in $\tau_i$ and
$m_{it}=1$ marks decision $t$ as a supervised target; the full visible
interaction remains context, and the mask only selects targets.
Two properties are required. \emph{Compositional coverage} (\chone):
the training data provide sufficient support for the observation--action
dependencies and multi-step structures the policy should acquire.
\emph{Decision grounding} (\chtwo): each supervised action is supported
by the realized interaction and by its decision-time context $c_t$.
Our goal is one procedure that constructs $\mathcal D_s$ with both
properties for every learner scale (Section~\ref{sec:method}).

\section{Method}
\label{sec:method}

\begin{figure*}[t]
  \centering
  \includegraphics[width=0.98\textwidth]{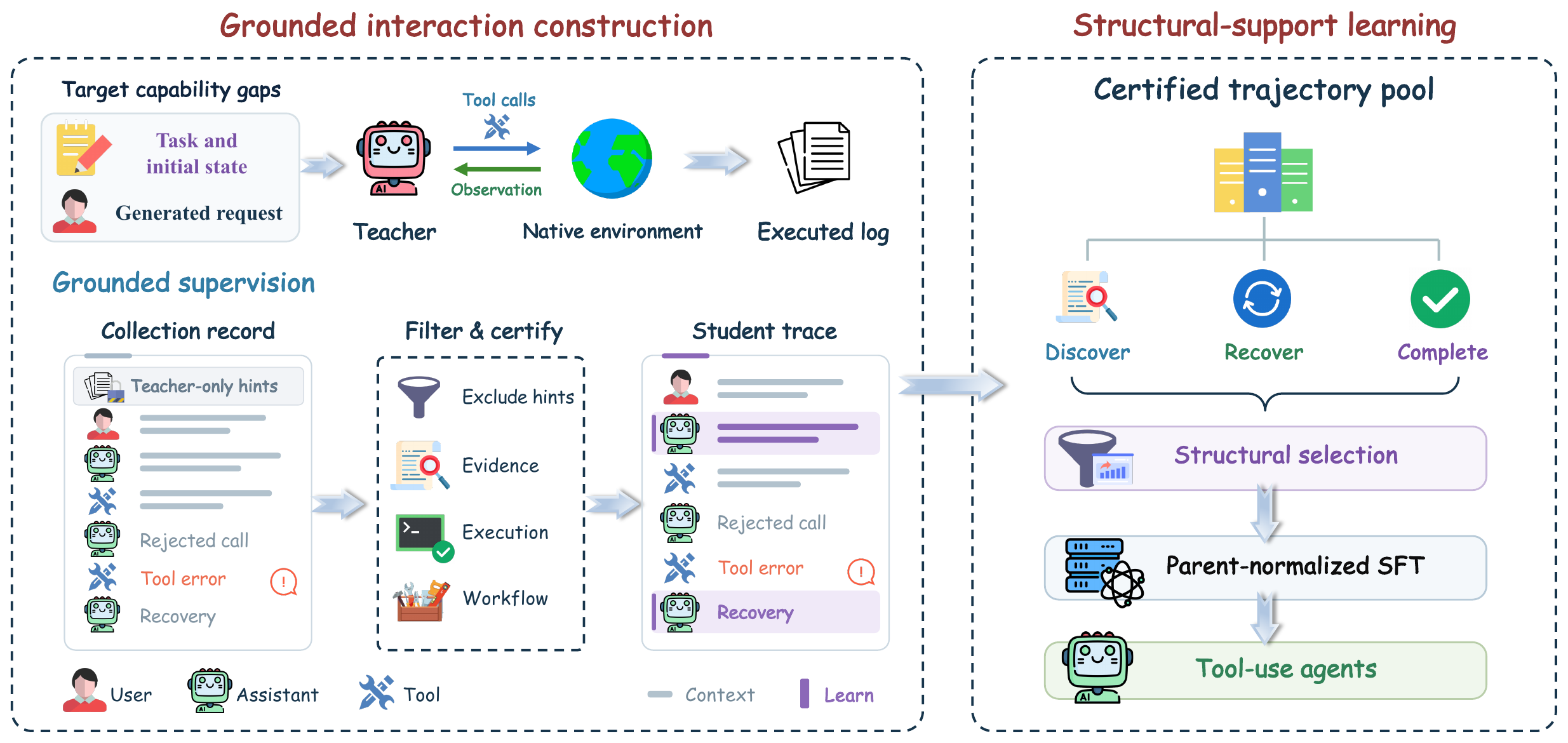}
  \caption{\textbf{Overview of \method{}.}
  % (1)~Compositional interaction synthesis builds tasks around capability
  % gaps and executes them in the native environment; (2)~grounded
  % supervision excludes teacher-only hints and certifies each decision by
  % visible evidence, execution, and workflow order, keeping authentic
  % errors as recovery context; (3)~structural-support learning selects
  % from the certified pool for parent-normalized SFT.
  }
  \label{fig:trajectory2gym-overall}
\end{figure*}

Figure~\ref{fig:trajectory2gym-overall} connects three stages.
(1)~Compositional interaction synthesis (Section~\ref{sec:synthesis},
for~\chone) turns under-supported paths into executed native
interactions; (2)~grounded supervision (Section~\ref{sec:certification},
for~\chtwo) turns each executed log into a student context and a set
of certified targets; and (3)~structural-support learning
(Section~\ref{sec:training}) decides how often each certified structure
enters training and how its rendered examples are weighted
(properties proved in Appendix~\ref{app:theory}).

\subsection{Compositional Interaction Synthesis}
\label{sec:data_pipeline}
\label{sec:synthesis}
\label{sec:capability_mining}

A shared capability vocabulary $\mathcal K$ covers semantic actions and
arguments, observation outcomes, recovery, branching, handoff,
termination, and horizon; native tools and observations map to these
roles, so one path can be supported by different backends. An ordered
path $p=(\kappa_1,\ldots,\kappa_\ell)$ composes these requirements within
one interaction. A generator $G$ constructs a request $x$ and compatible
initial state $s_0$ for a target path, and a teacher $\pi_{\mathrm T}$
executes the task in the native environment:
\begin{equation}
    (x,s_0)\sim G(\cdot\mid p,\env),
    \qquad
    \tau\sim P_{\env}(\cdot\mid x,s_0,\pi_{\mathrm T}).
    \label{eq:path-conditioned-generation}
\end{equation}
Each trajectory is assigned one canonical path $p(\tau)$ for support
accounting, and the support of $p$ in a selected set $\mathcal S$ is
\begin{equation}
    n_{\mathcal S}(p)
    =
    \sum_{\tau\in\mathcal S}
    \mathbf{1}[p(\tau)=p].
    \label{eq:path-support}
\end{equation}

The support scenario of Section~\ref{sec:introduction} is the path
lookup, error, clarification, retry, order read, confirmation,
cancellation, update.
We target paths whose support falls below the
scale's target (Section~\ref{sec:training}) and condition the generator
on them; rare structures need this explicit route, since uniform
selection of $B$ out of $N$ candidates misses a path represented once
with probability $1-B/N$ (Proposition~\ref{prop:random}). Construction
establishes the entities, prerequisites, and information sources each
path needs (a transaction task may identify the customer but require a
lookup for the transaction identifier; a recovery variant creates a
discoverable native rejection) and binds required values to visible
sources (the request, a user exchange, or a tool result), since only
such values can become grounded targets.

Native execution supplies authentic observations, state changes, and
rejections; teacher-only guidance stays out of the student trace. A path
is certified only when its action and observation dependencies appear in
the executed log, and only certified structures update the support
inventory and guide later synthesis rounds; a mistaken generator or
teacher proposal therefore surfaces as an uncertified structure, not a
training target (templates in Appendix~\ref{app:synthesis}).

\subsection{Grounded Supervision}
\label{sec:certification}

\stitle{Visible context.}
Grounded supervision turns the execution log into student contexts and
eligible targets, checking visibility, local execution, and workflow
completion in that order. The decision context $c_t$ in
Eq.~\eqref{eq:tool-use-policy} contains only events exposed through the
native agent interface; the full log remains available for
certification, but teacher hints and private events never enter the
student context.

\stitle{Target eligibility.}
We distinguish native execution validity $e_t\in\{0,1\}$ from
decision-time grounding $v_t\in\{0,1\}$, with local eligibility
\begin{equation}
    g_t=e_t\,v_t.
    \label{eq:local-eligibility}
\end{equation}
For a tool call, $e_t=1$ requires authenticated non-error execution or
native action-only validation; $v_t=1$ requires the tool to be available
to the agent and identifier-like arguments to be supported by preceding
visible information (discoverable tools must first appear in a visible
retrieval result). An identifier known only to the collector cannot
support a target. When the demonstrated action remains dependent on
collector-private information after conditioning on $c_t$, a policy
restricted to $c_t$ faces the additional uncertainty characterized by
Proposition~\ref{prop:grounding}. A rejected call and its authentic
error remain visible \emph{context} without imitation loss, while the
grounded recovery that follows becomes a target
(Figure~\ref{fig:trajectory2gym-grounded-recovery} in
Appendix~\ref{app:native-interaction-cases}).

\stitle{Workflow completion.}
Let $R_\tau$ be native terminal success, $q_\tau$ the task's ordered
requirements, and $A_\tau$ the canonicalized sequence of relevant user
operations and locally eligible assistant actions, with roles,
arguments, and order preserved. The final supervision mask is
\begin{equation}
    m_t
    =
    R_\tau\,
    \mathbf{1}[q_\tau\preceq A_\tau]\,
    g_t,
    \label{eq:certified-supervision}
\end{equation}
where $\preceq$ denotes ordered subsequence inclusion; other actions may
occur between required operations, and $q_\tau=\varnothing$ when the
native checker fully specifies the task. In the constructed contrast of
Figure~\ref{fig:trajectory2gym-teaser}, both calls are executable;
$v_t$ checks visible tool and argument provenance, not the later user's
task scope. When encoded by the native task, $R_\tau$ and $q_\tau$
check that scope and the required order. Rejected calls are masked but
their authentic errors remain visible context. Trajectories with
$\sum_t m_t>0$ enter the eligible pool; role-specific checks are in
Appendix~\ref{app:certification}.

\subsection{Structural-Support Learning}
\label{sec:training}
\label{sec:assembly}
\label{sec:selection}
\label{sec:rendering}
\label{sec:sft}

\stitle{Structural selection.}
We assemble $\mathcal D_s$ from the certified pool to preserve task
breadth, cover capability paths, and give repeated support to
under-represented structures such as discovery, recovery, and
completion. For scale $s$, let $d_s$ be the target number of distinct
trajectories per feasible canonical path; it grows with learner scale
(1, 2, 3, 4, and 8 from 1.7B to 30B-A3B), so larger learners see more
realizations of each path per epoch. The remaining deficit of a selected
set $\mathcal S$ is
\begin{equation}
    \delta_s(p;\mathcal S)
    =
    \bigl[d_s-n_{\mathcal S}(p)\bigr]_+.
    \label{eq:support-deficit}
\end{equation}
Selection starts from a fixed certified core, adds one completing
exemplar per training task, closes uncovered atoms and paths, and then
adds trajectories that reduce these deficits. Under a pure cardinality
budget, this deficit stage exactly maximizes capped support
$\sum_p\min(n_{\mathcal S}(p),d_s)$; the complete selector also has task
and per-task constraints (Proposition~\ref{prop:selection};
Appendix~\ref{app:selection}). Without the cap, the objective is modular
and indifferent to which paths are covered.

\stitle{Parent-normalized SFT.}
Selected trajectories are rendered through native prompts, roles, tool
schemas, and visible observations. One semantic trajectory may yield
several training examples; under per-example averaging a parent's weight
is proportional to its number of rendered children
(Proposition~\ref{prop:parent}), so long multi-turn interactions would
dominate the loss. We therefore define each trajectory's contribution at
the parent level. With $(h_u,y_u)_{u=1}^{N_\tau}$ the certified
next-token pairs of $\tau$ counted across rendered examples,
\begin{equation}
    \ell_\theta(\tau)
    =
    -\frac{1}{N_\tau}
    \sum_{u=1}^{N_\tau}
    \log\pi_\theta(y_u\mid h_u),
    \qquad
    \mathcal L_{\mathrm{SFT}}(\theta)
    =
    \frac{1}{|\mathcal D_s|}
    \sum_{\tau\in\mathcal D_s}
    \ell_\theta(\tau).
    \label{eq:sft-objective}
\end{equation}
Each parent thus receives equal weight, non-target visible messages
remain context, and regrouping targets into different children leaves
the objective unchanged (Appendix~\ref{app:parent-normalization}).

For illustration, consider a corpus of two interactions that render into
three examples and one example, respectively. Row averaging assigns them
total weights of $3/4$ and $1/4$, whereas parent normalization assigns
$1/2$ to each. Each parent's weight is distributed across its examples
in proportion to their certified target-token counts, keeping semantic
training mass independent of rendering multiplicity.

\section{Experiments}
\label{sec:experiments}

We ask whether data built for compositional coverage and decision
grounding improves tool use at every scale with SFT alone (\textbf{RQ1};
Section~\ref{sec:main-results}), how the resulting models compare with
public checkpoints and where the gains concentrate (\textbf{RQ2},
\textbf{RQ3}; Sections~\ref{sec:main-results}
and~\ref{sec:bfcl-results}), and which stages contribute under matched
ablations (\textbf{RQ4}; Section~\ref{sec:controlled});
Sections~\ref{sec:support-allocation} and~\ref{sec:qualitative-grounding}
add a depth sensitivity test and a same-task case study.

% Tables 1 and 2 sit one subsection early in the source so that each
% prints just before its discussion.
\begin{table}[t]
\caption{\textbf{Tau1 and Tau2 results at all five scales} (Pass@1, \%;
higher is better). Blue rows are \method{}-trained; $\Delta$ is the
paired gain in Mean, the equal average of the two benchmark averages.
Tau1 Airline is a
seen-task, transductive setting (Appendix~\ref{app:protocol}).
Agent-training checkpoints are re-evaluated on the same complete suites:
Simia \citep{li2025simia}, MUA-RL \citep{zhao2025muarl}, LoopTool
\citep{zhang2025looptool}, EnvFactory \citep{xu2026envfactory}, and MidTool-RL
\citep{jiang2026midtool} (all evaluated checkpoints in
Appendix~\ref{app:public-checkpoints}); closed-source rows are reported
references.}
\label{tab:current-tau}
\centering
\TgMainTableStyle
\setlength{\tabcolsep}{1.1pt}
\begin{tabularx}{\linewidth}{@{}l*{7}{N}R>{\raggedleft\arraybackslash}p{3.4em}@{}}
\toprule
\multirow{2}{*}{Model} & \multicolumn{3}{c}{Tau1}
 & \multicolumn{4}{c}{Tau2} & \multirow{2}{*}{\textbf{Mean}} & \multirow{2}{*}{$\Delta$} \\
\cmidrule(lr){2-4}\cmidrule(lr){5-8}
 & Retail & Airline & Avg. & Retail & Airline & Telecom & Avg. & & \\
\midrule
\tabsection{10}{Closed-source references}
GPT-5.2 (high) & 75.44 & 63.00 & 69.22 & 81.58 & 83.00 & 89.69 & 84.76 & 76.99 & \\
Claude-4.5-Opus (high) & 79.39 & 67.50 & 73.45 & 79.61 & 84.00 & 92.32 & 85.31 & 79.38 & \\
\midrule
\tabsection{10}{Re-evaluated agent-training checkpoints}
Simia Tau-SFT (Qwen3-8B) & 45.22 & 28.00 & 36.61 & 62.50 & 50.00 & 15.00 & 42.50 & 39.55 & \\
MUA-RL-8B & 51.30 & 18.00 & 34.65 & 55.00 & 15.00 & 7.50 & 25.83 & 30.24 & \\
MUA-RL-32B & 66.09 & 42.00 & 54.04 & 70.00 & 30.00 & 35.00 & 45.00 & 49.52 & \\
LoopTool-8B & 39.13 & 8.00 & 23.57 & 40.00 & 15.00 & 17.50 & 24.17 & 23.87 & \\
EnvFactory-8B (thinking) & 38.26 & 42.00 & 40.13 & 42.50 & 25.00 & 12.50 & 26.67 & 33.40 & \\
MidTool-RL-8B & 41.74 & 8.00 & 24.87 & 57.50 & 10.00 & 20.00 & 29.17 & 27.02 & \\
\midrule
\tabsection{10}{Qwen3 initialization and \method{} post-training}
Qwen3-1.7B & 13.04 & 4.00 & 8.52 & 10.00 & 20.00 & 40.00 & 23.33 & 15.93 & \\
\rowcolor{T2GTableBlue}
\textbf{\quad + \method{}} & 40.87 & 28.00 & 34.43 & 42.50 & 20.00 & 50.00 & 37.50 & \textbf{35.97} & $+20.04$ \\
Qwen3-4B & 30.43 & 18.00 & 24.22 & 20.00 & 15.00 & 22.50 & 19.17 & 21.70 & \\
\rowcolor{T2GTableBlue}
\textbf{\quad + \method{}} & 50.43 & 38.00 & 44.22 & 72.50 & 50.00 & 85.00 & 69.17 & \textbf{56.70} & $+35.00$ \\
Qwen3-8B & 41.74 & 14.00 & 27.87 & 45.00 & 10.00 & 20.00 & 25.00 & 26.44 & \\
\rowcolor{T2GTableBlue}
\textbf{\quad + \method{}} & 58.26 & 44.00 & 51.13 & 77.50 & 50.00 & 95.00 & 74.17 & \textbf{62.65} & $+36.21$ \\
Qwen3-14B & 46.09 & 18.00 & 32.04 & 40.00 & 20.00 & 15.00 & 25.00 & 28.52 & \\
\rowcolor{T2GTableBlue}
\textbf{\quad + \method{}} & 66.09 & 50.00 & 58.04 & 85.00 & 75.00 & 82.50 & 80.83 & \textbf{69.44} & $+40.92$ \\
Qwen3-30B-A3B & 54.78 & 34.00 & 44.39 & 57.50 & 15.00 & 12.50 & 28.33 & 36.36 & \\
\rowcolor{T2GTableBlue}
\textbf{\quad + \method{}} & 69.57 & 64.00 & 66.78 & 87.50 & 70.00 & 95.00 & 84.17 & \textbf{75.47} & $+39.11$ \\
\bottomrule
\end{tabularx}
\end{table}

\subsection{Experimental Setup}

We evaluate \method{} on Tau1, Tau2, and BFCL with five Qwen3
scales (1.7B, 4B, 8B, 14B, and 30B-A3B). Each model starts from its public
initialization and receives one epoch of full-parameter, parent-normalized
SFT on its scale-conditioned view (11.6k to 19.9k trajectories;
Appendix~\ref{app:data-statistics}). Tau1 and Tau2 are scored by native
task success with equal-weight domain averaging; BFCL reports its eight
categories and the three-group Overall of Gecko.
The evidence standard is the paired comparison: each post-trained model
is compared with its own initialization under the same task population,
native interface, decoding, user simulator, and scorer, with each task
run once. We additionally evaluate public agent-training checkpoints
from five families, trained through trajectory synthesis, supervised learning,
or interactive learning, on the complete Tau1, Tau2, and BFCL suites with
the same task populations and scorers and each model's native prompting
and tool-call format (Appendix~\ref{app:public-checkpoints}); the tables
below show one 8B-scale checkpoint from each of five families, plus
MUA-RL-32B.
Closed-source systems and ToolACE-2-8B are reported references
(Appendix~\ref{app:published-refs}). Neither group is a matched or
equal-compute comparison: training data, model families, and context
budgets differ, which precludes component-level attribution across
methods.

\begin{table}[t]
\caption{\textbf{BFCL results} (\%; higher is better). Overall equally
averages the four-category Non-live mean, the three-category Live mean,
and Multi-turn Base. Blue rows are \method{}-trained; $\Delta$ is the
paired gain in Overall. Closed-source rows and ToolACE-2-8B follow Gecko
Table~3 \citep{zhang2026gecko} without GATS test-time refinement; the
other agent-training checkpoints are re-evaluated as in
Table~\ref{tab:current-tau}. Par., Mult., and Irrel. abbreviate the
Parallel, Multiple, and Irrelevance categories; a Parallel score of 0.00
is 0/100 correct, not missing data.}
\label{tab:current-bfcl}
\centering
\TgMainTableStyle
\setlength{\tabcolsep}{0.7pt}
\begin{tabularx}{\linewidth}{@{}l*{8}{N}>{\raggedleft\arraybackslash}p{4.0em}>{\raggedleft\arraybackslash}p{3.4em}@{}}
\toprule
\multirow{2}{*}{Model} & \multicolumn{4}{c}{Non-live} & \multicolumn{3}{c}{Live}
 & Multi-turn & \multirow{2}{*}{\textbf{Overall}} & \multirow{2}{*}{$\Delta$} \\
\cmidrule(lr){2-5}\cmidrule(lr){6-8}\cmidrule(lr){9-9}
 & Simple & Par. & Mult. & Irrel. & Simple & Mult. & Irrel. & Base & & \\
\midrule
\tabsection{11}{Closed-source references}
GPT-5.5 & 91.50 & 90.50 & 88.50 & 88.75 & 73.64 & 73.79 & 83.82 & 60.00 & 75.63 &  \\
Gemini-3.0-Pro & 94.50 & 91.00 & 94.00 & 82.50 & 87.60 & 80.44 & 73.19 & 69.00 & 79.97 &  \\
\midrule
\tabsection{11}{Agent-training references}
ToolACE-2-8B & 88.00 & 92.50 & 92.50 & 95.41 & 70.93 & 79.01 & 84.80 & 49.00 & 73.12 &  \\
Simia Tau-SFT (Qwen3-8B) & 88.50 & 0.00 & 92.00 & 85.83 & 58.14 & 38.14 & 88.69 & 1.00 & 43.08 &  \\
MUA-RL-8B & 94.50 & 0.00 & 94.00 & 81.67 & 74.42 & 72.30 & 84.16 & 30.00 & 58.17 &  \\
MUA-RL-32B & 94.50 & 66.00 & 96.00 & 84.17 & 84.50 & 72.87 & 86.43 & 51.00 & 72.48 &  \\
LoopTool-8B & 97.50 & 97.00 & 96.00 & 81.67 & 83.72 & 75.90 & 74.21 & 35.00 & 68.66 &  \\
EnvFactory-8B (thinking) & 89.50 & 96.00 & 96.00 & 75.00 & 85.27 & 76.85 & 64.48 & 68.00 & 77.55 &  \\
MidTool-RL-8B & 96.00 & 85.00 & 93.00 & 74.17 & 86.82 & 73.06 & 68.55 & 58.00 & 73.73 &  \\
\midrule
\tabsection{11}{Qwen3 initialization and \method{} post-training}
Qwen3-1.7B & 91.00 & 84.00 & 93.00 & 85.83 & 64.34 & 61.86 & 75.34 & 12.00 & 55.88 &  \\
\rowcolor{T2GTableBlue}
\textbf{\quad + \method{}} & 85.50 & 79.00 & 90.00 & 88.33 & 53.49 & 68.12 & 75.57 & 48.00 & \textbf{66.48} & $+10.60$ \\
Qwen3-4B & 93.50 & 92.00 & 95.00 & 86.67 & 76.74 & 72.87 & 78.28 & 21.00 & 62.92 &  \\
\rowcolor{T2GTableBlue}
\textbf{\quad + \method{}} & 92.50 & 90.00 & 97.00 & 85.83 & 68.22 & 74.19 & 71.27 & 62.00 & \textbf{74.85} & $+11.93$ \\
Qwen3-8B & 95.50 & 94.00 & 97.00 & 83.33 & 73.64 & 73.81 & 72.85 & 36.00 & 67.30 &  \\
\rowcolor{T2GTableBlue}
\textbf{\quad + \method{}} & 95.00 & 94.00 & 97.00 & 89.17 & 74.42 & 74.95 & 72.85 & 68.00 & \textbf{78.62} & $+11.32$ \\
Qwen3-14B & 93.00 & 95.00 & 93.00 & 90.00 & 79.84 & 76.66 & 75.34 & 46.00 & 72.01 &  \\
\rowcolor{T2GTableBlue}
\textbf{\quad + \method{}} & 93.50 & 94.00 & 96.00 & 90.00 & 73.64 & 79.32 & 83.03 & 67.00 & \textbf{79.68} & $+7.67$ \\
Qwen3-30B-A3B & 94.00 & 94.00 & 97.00 & 80.83 & 82.95 & 74.19 & 71.72 & 65.00 & 77.58 &  \\
\rowcolor{T2GTableBlue}
\textbf{\quad + \method{}} & 96.50 & 94.00 & 98.00 & 89.17 & 69.77 & 76.09 & 83.26 & 72.00 & \textbf{80.93} & $+3.35$ \\
\bottomrule
\end{tabularx}
\end{table}

\subsection{Stateful Tool Use (RQ1 to RQ3)}
\label{sec:main-results}

\stitle{\method{} raises the Tau1 and Tau2 mean by 20 to 41 points at
every scale.}
All five initializations improve on both benchmarks
(Table~\ref{tab:current-tau}); the paired gain is smallest at 1.7B
(20.04) and between 35 and 41 points from 4B upward, and supervision
changes what a given capacity delivers: the post-trained 4B model reaches
69.17 on Tau2, against 28.33 for the public 30B-A3B initialization, and
the post-trained 30B-A3B model reaches 84.17, within 1.2 points of the
reported closed-source references, 84.76 and 85.31 (RQ1).
Against re-evaluated agent-training checkpoints, our 8B model reaches
74.17 on Tau2, at least 28 points above every one of them, 14B and
32B ones included, and on Tau1 trails only MUA-RL-32B, which our 14B and
30B-A3B models exceed (RQ2). The gain is also broad: across five domains
and five scales, 24 of 25 paired comparisons improve, with the largest
movement in Tau2 Telecom, which rises from 12.50 to 95.00 at 30B-A3B,
and the exception, Tau2 Airline at 1.7B, is unchanged at 20.00 (RQ3).
Taken together, standard SFT on decision-grounded compositional data,
with no reinforcement learning, makes every initialization from 4B
upward a stronger stateful agent than any evaluated checkpoint,
RL-trained and 32B ones included (Tau mean 56.70 to 75.47 versus at most
49.52); the margin is widest on Tau2, whose multi-step, state-dependent
tasks are where compositional coverage and decision grounding are meant
to act. These gains reflect the complete recipe, whose training view
also changes with scale (Appendix~\ref{app:data-statistics}); the
matched ablations of Section~\ref{sec:controlled} separate its stages.

\subsection{Function Calling and Multi-Turn Behavior (RQ2, RQ3)}
\label{sec:bfcl-results}

\stitle{BFCL gains concentrate in multi-turn tool use.}
Overall improves at every scale, driven by Multi-turn Base, which rises
by 7 to 41 points, consistent with the cross-step dependencies that
synthesis targets (RQ3): every post-trained model from 4B upward exceeds
the reported GPT-5.5 (60.00) on Multi-turn Base, and our 30B-A3B model
posts 80.93 Overall, numerically above the 79.97 reported for
Gemini-3.0-Pro under its source protocol. Single-turn categories start
high and mostly hold; the one consistent exception is Live Simple, down 6 to 13 points
at four of five scales, which we hypothesize reflects a shift toward
information gathering on queries that expect one immediate call.
Function-calling skill does not transfer to stateful tasks, which also
require later actions to stay consistent with accumulated evidence and
state: EnvFactory-8B nearly matches our 8B model on BFCL Overall (77.55
versus 78.62) and ties it on Multi-turn Base (68.00), yet trails it by
47.5 points on Tau2 (Table~\ref{tab:current-tau}). Across all three
benchmarks, our 8B model therefore has the highest equal-weight mean of
any evaluated checkpoint, 67.97 against at most 57.17 (RQ2).

\subsection{Ablation Study (RQ4)}
\label{sec:controlled}
\label{sec:ablation}

\begin{table}[t]
\centering
\caption{\textbf{Matched component ablations on Qwen3-8B} (\%).
Each pair compares one \method{} design choice with its stage-matched
alternative.
Mean equally averages Tau1, Tau2, and BFCL; Gain is the \method{} choice
minus its alternative, so positive values favor the proposed component.
Matching conditions are given in Appendix~\ref{app:controls}.}
\label{tab:ablation}
\label{tab:controlled-analysis}
\small
\setlength{\tabcolsep}{4.2pt}
\renewcommand{\arraystretch}{0.95}
\begin{tabular*}{\linewidth}{@{\extracolsep{\fill}}lrrrrr@{}}
\toprule
Setting & Tau1 & Tau2 & BFCL & Mean & Gain \\
\midrule
\rowcolor{GroupGray}
\multicolumn{6}{@{}l}{\textit{Compositional interaction synthesis}} \\
Existing candidate pool & 51.00 & 75.83 & 80.39 & 69.08 & -- \\
\rowcolor{T2GTableBlue}
\textbf{+ Targeted expansion} & \textbf{58.61} & \textbf{81.67}
& 77.61 & \textbf{72.63} & \textbf{+3.55} \\
\rowcolor{GroupGray}
\multicolumn{6}{@{}l}{\textit{Decision-grounded supervision}} \\
Loose success-derived targets & 48.13 & 71.67 & 75.86 & 65.22 & -- \\
\rowcolor{T2GTableBlue}
\textbf{+ Decision grounding} & \textbf{51.13} & \textbf{74.17}
& \textbf{78.62} & \textbf{67.97} & \textbf{+2.75} \\
\rowcolor{GroupGray}
\multicolumn{6}{@{}l}{\textit{Structural-support selection}} \\
Random selection & 51.13 & 74.17 & \textbf{79.40} & 68.23 & -- \\
\rowcolor{T2GTableBlue}
\textbf{+ Structural selection} & \textbf{57.00} & \textbf{75.83}
& 78.81 & \textbf{70.55} & \textbf{+2.32} \\
\rowcolor{GroupGray}
\multicolumn{6}{@{}l}{\textit{Parent-normalized training mass}} \\
Rendered-row averaging & 49.52 & 52.50 & 77.79 & 59.94 & -- \\
\rowcolor{T2GTableBlue}
\textbf{+ Parent normalization} & \textbf{51.13} & \textbf{74.17}
& \textbf{78.62} & \textbf{67.97} & \textbf{+8.04} \\
\bottomrule
\end{tabular*}
\end{table}

\begin{figure}[t]
  \centering
  \begin{minipage}[c]{0.5\textwidth}
    \includegraphics[width=\linewidth]{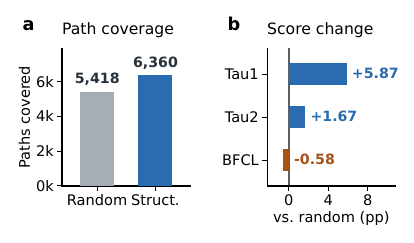}
  \end{minipage}\hfill
  \begin{minipage}[c]{0.47\textwidth}
    \caption{\textbf{Structural vs.\ random selection at a fixed 8B
    budget} (14,819 parents, equal target tokens). (a) Canonical paths
    covered. (b) Score change of structural selection (points).
    Structural selection covers 942 more paths and improves both Tau
    benchmarks, with BFCL essentially unchanged.}
    \label{fig:fixed8b-diagnostics}
    \label{fig:app-structural-selection}
  \end{minipage}
\end{figure}

\stitle{All four 8B matched comparisons favor the proposed choices.}
Each pair in Table~\ref{tab:ablation} changes one design choice at
\mbox{Qwen3-8B} and keeps the rest of the procedure matched as closely as
possible (Appendix~\ref{app:controls}); each arm is a single run, and one
Tau2 Airline task moves the three-benchmark mean by 0.56 points, so gains
of a few points are directional. All four choices raise the mean in their
stage-matched comparisons. Targeted expansion mainly strengthens the
stateful Tau tasks, but its pair changes candidate supply without fixing
the target-token budget, so it does not isolate targeting at fixed data
volume. Grounding improves decision-level supervision, and structural
selection covers 942 more canonical paths at the same budget
(Figure~\ref{fig:fixed8b-diagnostics}). Parent normalization gives
the largest isolated gain (8.04 points): the semantic trajectories,
rendered examples, and target masks are unchanged, so this pair isolates
how optimization mass is distributed and supports treating a semantic
interaction, rather than a rendered row, as the unit of optimization.
Because each pair intervenes at a different stage, the gains are not
additive.

\subsection{Sensitivity Analysis}
\label{sec:support-allocation}

\begin{table}[!ht]
\centering
\caption{\textbf{Sensitivity to structural-support depth at Qwen3-8B} (\%).
Depth controls the amount of repeated support allocated to certified
interaction structures. These control-specific runs use a different
training view from the headline 8B checkpoint of
Tables~\ref{tab:current-tau} and~\ref{tab:current-bfcl}.}
\label{tab:support-depth}
\label{tab:app-fixed8b-depth}
\small
\setlength{\tabcolsep}{5pt}
\renewcommand{\arraystretch}{1.10}
\begin{tabular*}{\linewidth}{@{\extracolsep{\fill}}crrrrrr@{}}
\toprule
Depth & Trajectories & Target tokens & Tau1 & Tau2 & BFCL & Mean \\
\midrule
2 & 13,986 & 18.02M & 55.61 & \textbf{82.50} & 78.74 & \textbf{72.28} \\
3 & 14,819 & 19.42M & \textbf{57.00} & 75.83 & 78.81 & 70.55 \\
4 & 15,805 & 20.61M & 54.00 & 75.00 & \textbf{79.31} & 69.44 \\
\bottomrule
\end{tabular*}
\end{table}

\stitle{More structural support is not automatically better.}
Performance does not increase monotonically with support depth
(Table~\ref{tab:support-depth}):
the three-benchmark mean is highest at the shallowest setting (72.28
versus 70.55 and 69.44) even as trajectories and target tokens continue to
grow. The decline is concentrated in Tau2 (82.50 to 75.00), while BFCL
remains comparatively stable. The headline schedule was fixed before these
runs, so this is a sensitivity test rather than checkpoint selection;
because depth also changes data volume, it does not isolate path support
from token count. In short, structural support is an allocation problem
rather than a simple data-scaling rule, and we set depth per scale by a
predeclared schedule (Appendices~\ref{app:support-analysis}
and~\ref{app:selection}).

The structural-selection control and the depth sweep answer complementary
questions. The former changes the distribution of represented structures
at matched parent and target-token budgets; the latter increases
repetition together with supervision volume. A gain from the first
intervention can therefore coexist with a decline from the second.
In terms of the construction objective, adding an uncovered path expands
breadth, whereas adding another interaction to a covered path increases
support. This distinction links the empirical results to the separate
coverage and deficit-allocation stages of the selector.

\subsection{Qualitative Analysis}
\label{sec:qualitative-grounding}

Figure~\ref{fig:trajectory2gym-baseline-comparison} contrasts three 8B
agents on a Retail workflow whose two updates must occur in order.
Simia has seen the requested address in an order record, yet keeps
querying the user profile and transfers to a human: exposed evidence
never reaches the decision that needs it. MUA-RL holds the right address
but edits the item first, which changes the order status and blocks the
address update, an ordering error invisible to per-call checks because
the edit itself executes cleanly. \method{} confirms both changes,
updates the address while the order is pending, then edits the items.
These are exactly the behaviors our construction supervises: grounding
admits only targets whose arguments trace to visible evidence, and
synthesis with workflow-order certification builds and certifies such
dependent-state paths. Appendix~\ref{app:qualitative} gives the traces
and six further paired cases.

The recovery example in Appendix~\ref{app:native-interaction-cases}
complements this ordering case.
After an email lookup returns ``user not found,'' the agent obtains the
user's name and ZIP code and resumes the workflow through an alternative
lookup.
The failed call and its authentic response remain visible context, while
the subsequent grounded decisions supply supervision.
This preserves the reason for recovery without assigning imitation loss
to the failed attempt.
Together, the two cases illustrate complementary dependencies: preserving
a prerequisite before changing state, and using observed feedback to
choose the next action.
Both require the training representation to retain the policy-visible
sequence rather than reduce the interaction to a collection of successful
calls.
A successful call may acquire new evidence, satisfy a prerequisite, or
complete a recovery; its training role therefore depends on its place
in the interaction, not execution validity alone.

The cases in Appendix~\ref{app:case-studies} cover travel and
connectivity workflows.
Returned booking identifiers support insurance, invoice, or support
operations; user-reported device state and connection speed guide
connectivity recovery.
These dependencies connect decisions through observations and user-mediated
state changes, rather than tool names alone.

\begin{figure}[!ht]
  \centering
  \includegraphics[width=\linewidth]{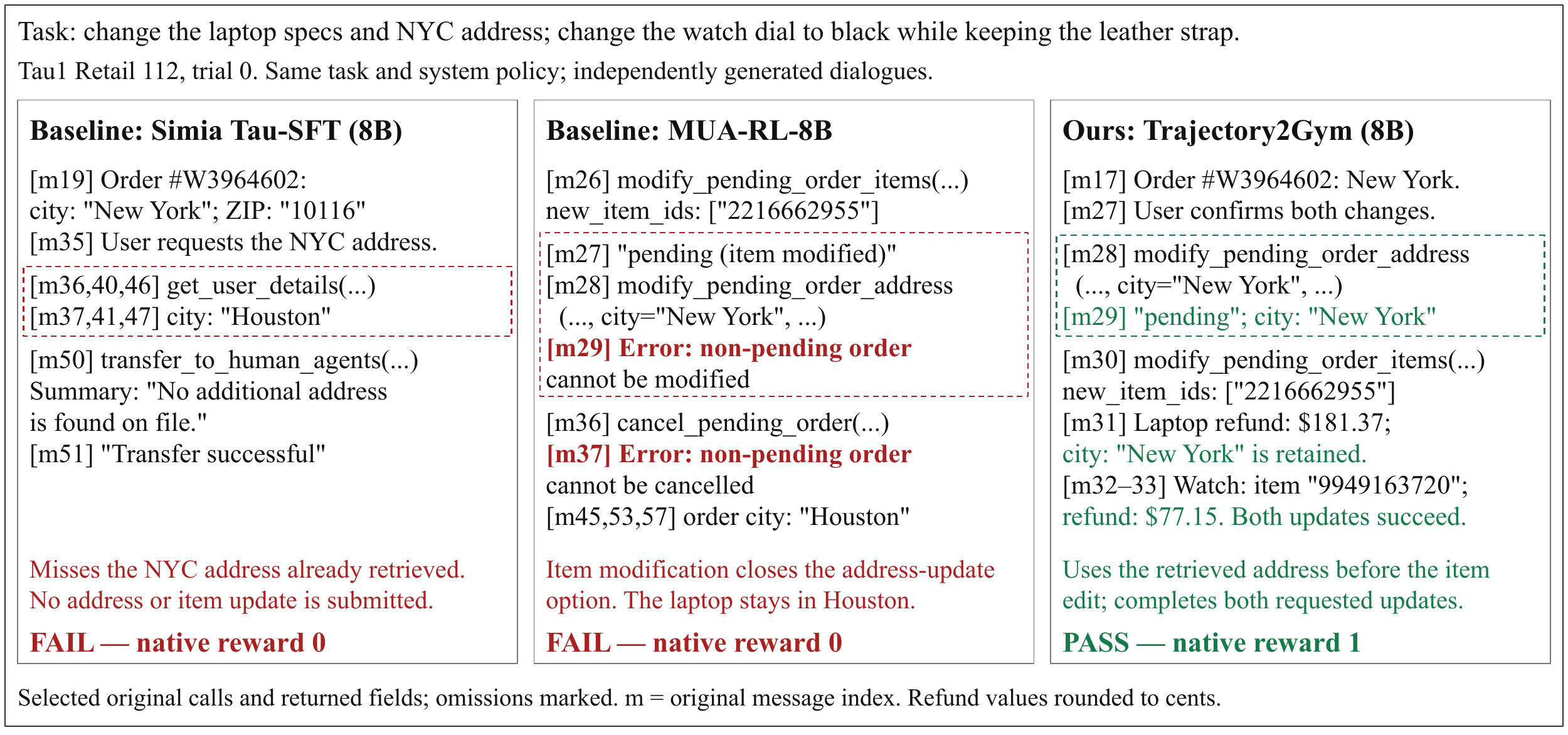}
  \caption{\textbf{Same-task comparison of evidence use and action order.}
  Simia overlooks visible evidence; MUA-RL violates a state-dependent ordering;
  \method{} reuses the evidence and preserves the required dependency
  (native rewards 0, 0, and 1).
  Interactions are independently generated and abridged for display;
  a step-by-step account appears in
  Appendix~\ref{app:baseline-case-comparison}.}
  \label{fig:trajectory2gym-baseline-comparison}
\end{figure}
\FloatBarrier

\section{Conclusion}
\label{sec:conclusion}

Successful trajectories are the default supervision for tool-use agents,
yet success certifies an episode, not what it teaches. This paper has
argued that a trajectory's value lies in properties an outcome filter
cannot see, compositional coverage and decision grounding, and \method{}
builds training data around them: it synthesizes the interactions a pool
lacks, certifies supervision one decision at a time, and allocates
structural support per scale. Across scales from 1.7B to 30B, standard
SFT on the constructed interactions improves tool use, especially on
multi-turn and stateful tasks. Matched 8B controls give directional
evidence for the four design choices, with the clearest isolated gain
from parent normalization. The broader lesson is that success should
gate supervision, not license it;
that authentic errors belong in context, not in targets; and that
support is an allocation, not a quantity to maximize.

The resulting training corpus can be viewed at three levels: which
interaction structures are available, which decisions are eligible for
imitation, and how much optimization mass each complete interaction
receives. These levels can vary even when two datasets contain similar
numbers of trajectories or target tokens. Treating them separately makes
data-construction choices explicit and helps interpret why larger
collections do not automatically provide better supervision. For future
tool-use training pipelines, this suggests tracking structural coverage
and decision eligibility alongside data volume, while keeping the
training weight of an interaction independent of its rendering format.

Extending decision-level certification to on-policy rollouts and to web, software,
and terminal agents is the natural next step. What a trajectory teaches
is decided not by its outcome, but by the dependencies it realizes and
the decisions its visible history can justify.

\clearpage
\bibliography{references}

@inproceedings{xi2025agentgymrl,
  title = {{AgentGym-RL}: An Open-Source Framework to Train {LLM} Agents for Long-Horizon Decision Making via Multi-Turn {RL}},
  author = {Zhiheng Xi and Jixuan Huang and Chenyang Liao and Baodai Huang and Jiaqi Liu and Honglin Guo and Yajie Yang and Rui Zheng and Junjie Ye and Jiazheng Zhang and Wenxiang Chen and Wei He and Yiwen Ding and Guanyu Li and Zehui Chen and Zhengyin Du and Xuesong Yao and Yufei Xu and Jiecao Chen and Tao Gui and Zuxuan Wu and Qi Zhang and Xuanjing Huang and Yu-Gang Jiang},
  booktitle = {The Fourteenth International Conference on Learning Representations},
  year = {2026}
}

@inproceedings{qian2025toolrl,
  title = {{ToolRL}: Reward is All Tool Learning Needs},
  author = {Qian, Cheng and Acikgoz, Emre Can and He, Qi and Wang, Hongru and Chen, Xiusi and Hakkani-T{\"u}r, Dilek and Tur, Gokhan and Ji, Heng},
  booktitle = {Advances in Neural Information Processing Systems 38},
  pages = {116896--116926},
  year = {2025},
  publisher = {Neural Information Processing Systems Foundation, Inc. (NeurIPS)},
  doi = {10.52202/085713-3524}
}

@article{tian2026astra,
  title = {{ASTRA}: Automated Synthesis of Agentic Trajectories and Reinforcement Arenas},
  author = {Xiaoyu Tian and Haotian Wang and Shuaiting Chen and Hao Zhou and Kaichi Yu and Yudian Zhang and Jade Ouyang and Junxi Yin and Jiong Chen and Baoyan Guo and Lei Zhang and Junjie Tao and Yuansheng Song and Ming Cui and Chengwei Liu},
  journal = {arXiv preprint arXiv:2601.21558},
  year = {2026}
}

@inproceedings{wang2026agentworldmodel,
  title = {Agent World Model: Infinity Synthetic Environments for Agentic Reinforcement Learning},
  author = {Zhaoyang Wang and Canwen Xu and Boyi Liu and Yite Wang and Siwei Han and Zhewei Yao and Huaxiu Yao and Yuxiong He},
  booktitle = {Forty-third International Conference on Machine Learning},
  year = {2026}
}

@inproceedings{wang2026trajectory2task,
  title        = {{Trajectory2Task}: Training Robust Tool-Calling Agents with Synthesized Yet Verifiable Data for Complex User Intents},
  author       = {Wang, Ziyi and Lu, Yuxuan and Zhang, Yimeng and Chen, Pei and Dong, Ziwei and Huang, Jing and Gesi, Jiri and Tang, Xianfeng and Luo, Chen and Liu, Qun and Sang, Yisi and Lu, Hanqing and Li, Manling and Lai, Jin and Wang, Dakuo},
  booktitle    = {Proceedings of the 64th Annual Meeting of the Association for Computational Linguistics (Volume 1: Long Papers)},
  pages        = {44021--44044},
  year         = {2026},
  publisher    = {Association for Computational Linguistics},
  doi          = {10.18653/v1/2026.acl-long.2037}
}

@inproceedings{mialon2024gaia,
  title = {{GAIA}: A Benchmark for General {AI} Assistants},
  author = {Gr{\'e}goire Mialon and Cl{\'e}mentine Fourrier and Thomas Wolf and Yann LeCun and Thomas Scialom},
  booktitle = {The Twelfth International Conference on Learning Representations},
  year = {2024}
}

@article{luo2025mcpuniverse,
  title        = {{MCP-Universe}: Benchmarking Large Language Models with Real-World Model Context Protocol Servers},
  author       = {Luo, Ziyang and Shen, Zhiqi and Yang, Wenzhuo and Zhao, Zirui and Jwalapuram, Prathyusha and Saha, Amrita and Sahoo, Doyen and Savarese, Silvio and Xiong, Caiming and Li, Junnan},
  journal      = {arXiv preprint arXiv:2508.14704},
  year         = {2025}
}

@article{mo2025livemcpbench,
  title = {{LiveMCPBench}: Can Agents Navigate an Ocean of {MCP} Tools?},
  author = {Guozhao Mo and Wenliang Zhong and Jiawei Chen and Qianhao Yuan and Xuanang Chen and Yaojie Lu and Hongyu Lin and Ben He and Xianpei Han and Le Sun},
  journal = {arXiv preprint arXiv:2508.01780},
  year = {2025}
}

@inproceedings{prabhakar2025apigenmt,
  title = {{APIGen-MT}: Agentic Pipeline for Multi-Turn Data Generation via Simulated Agent-Human Interplay},
  author = {Prabhakar, Akshara and Liu, Zuxin and Zhu, Ming and Zhang, Jianguo and Awalgaonkar, Tulika Manoj and Wang, Shiyu and Liu, Zhiwei and Chen, Haolin and Hoang, Thai and Niebles, Juan Carlos and Heinecke, Shelby and Yao, Weiran and Wang, Huan and Savarese, Silvio and Xiong, Caiming},
  booktitle = {Advances in Neural Information Processing Systems 38},
  pages = {72647--72669},
  year = {2025},
  publisher = {Neural Information Processing Systems Foundation, Inc. (NeurIPS)},
  doi = {10.52202/085713-2184}
}

@article{xu2025toucan,
  title        = {{TOUCAN}: Synthesizing 1.5M Tool-Agentic Data from Real-World {MCP} Environments},
  author       = {Xu, Zhangchen and Meza Soria, Adriana and Tan, Shawn and Roy, Anurag and Agrawal, Ashish Sunil and Poovendran, Radha and Panda, Rameswar},
  journal      = {arXiv preprint arXiv:2510.01179},
  year         = {2025}
}

@article{xu2026envfactory,
  title        = {{EnvFactory}: Scaling Tool-Use Agents via Executable Environments Synthesis and Robust {RL}},
  author       = {Xu, Minrui and Wang, Zilin and Deng, Mengyi and Li, Zhiwei and Yang, Zhicheng and Zhu, Xiao and Liu, Yinhong and Zhu, Boyu and Huang, Baiyu and Chen, Chao and Deng, Heyuan and Mi, Fei and Shang, Lifeng and Zeng, Xingshan and Guo, Zhijiang},
  journal      = {arXiv preprint arXiv:2605.18703},
  year         = {2026}
}

@inproceedings{zhou2023webarena,
  title = {{WebArena}: A Realistic Web Environment for Building Autonomous Agents},
  author = {Shuyan Zhou and Frank F. Xu and Hao Zhu and Xuhui Zhou and Robert Lo and Abishek Sridhar and Xianyi Cheng and Tianyue Ou and Yonatan Bisk and Daniel Fried and Uri Alon and Graham Neubig},
  booktitle = {The Twelfth International Conference on Learning Representations},
  year = {2024}
}

@inproceedings{trivedi2024appworld,
  title        = {{AppWorld}: A Controllable World of Apps and People for Benchmarking Interactive Coding Agents},
  author       = {Trivedi, Harsh and Khot, Tushar and Hartmann, Mareike and Manku, Ruskin and Dong, Vinty and Li, Edward and Gupta, Shashank and Sabharwal, Ashish and Balasubramanian, Niranjan},
  booktitle    = {Proceedings of the 62nd Annual Meeting of the Association for Computational Linguistics (Volume 1: Long Papers)},
  pages        = {16022--16076},
  year         = {2024},
  publisher    = {Association for Computational Linguistics},
  doi          = {10.18653/v1/2024.acl-long.850}
}

@inproceedings{yao2024taubench,
  title = {{$\tau$-bench}: A Benchmark for Tool-Agent-User Interaction in Real-World Domains},
  author = {Shunyu Yao and Noah Shinn and Pedram Razavi and Karthik R Narasimhan},
  booktitle = {The Thirteenth International Conference on Learning Representations},
  year = {2025}
}

@inproceedings{barres2025tau2bench,
  title = {{$\tau^2$-Bench}: Evaluating Conversational Agents in a Dual-Control Environment},
  author = {Victor Barres and Honghua Dong and Soham Ray and Xujie Si and Karthik R Narasimhan},
  booktitle = {Forty-third International Conference on Machine Learning},
  year = {2026}
}

@inproceedings{shi2026tauknowledge,
  title = {{$\tau$-Knowledge}: Evaluating Conversational Agents over Unstructured Knowledge},
  author = {Quan Shi and Alexandra Zytek and Pedram Razavi and Karthik R Narasimhan and Victor Barres},
  booktitle = {Forty-third International Conference on Machine Learning},
  year = {2026}
}

@article{zhang2026gecko,
  title        = {{Gecko}: A Simulation Environment with Stateful Feedback for Refining Agent Tool Calls},
  author       = {Zhang, Zeyu and Li, Guohao and Xing, Zhenchang and Apostolopoulos, Alexandros and Lee, Yu Lin and Zheng, Liang},
  journal      = {arXiv preprint arXiv:2602.19218},
  year         = {2026}
}

@article{hu2026seal,
  title        = {{SEAL}: Synergistic Co-Evolution of Agents and Learning Environments},
  author       = {Hu, Yihao and Wen, Zhihao and Liu, Xiujin and Wang, Pan and Zhang, Xin and Wu, Wei},
  journal      = {arXiv preprint arXiv:2605.24426},
  year         = {2026}
}

@article{wang2026diagnosis,
  title        = {Diagnosis Before Recovery: Turning Agent Failures into Selective Self-Correction},
  author       = {Wang, Pan and Hu, Yihao and Wang, Hang and Lv, Zirui and Zhang, Xin and Li, Jianshe and Yang, Jiang-Ming and Wu, Wei and Tong, Yongqi},
  journal      = {arXiv preprint arXiv:2608.11772},
  year         = {2026}
}

@article{tong2026grounded,
  title        = {Grounded Skill Synthesis from Code at Scale for Agentic Intelligence},
  author       = {Tong, Yongqi and Wang, Pan and Wang, Hang and Li, Jianshe and Zhang, Xin and Yang, Jiang-Ming and Wu, Wei},
  journal      = {arXiv preprint arXiv:2609.05571},
  year         = {2026}
}

@article{zhao2026towards,
  title        = {Towards Better Agents for Multi-Turn User Interaction: The Next User Turn Is More Than Context},
  author       = {Zhao, Yiwen and Wen, Zhihao and Mao, Yuchen and Jiang, Mingxuan and Hu, Yihao and Wang, Pan and Zhang, Xin and Wu, Wei},
  journal      = {arXiv preprint arXiv:2608.17499},
  year         = {2026}
}

@inproceedings{liu2025toolace,
  title = {{ToolACE}: Winning the Points of {LLM} Function Calling},
  author = {Weiwen Liu and Xu Huang and Xingshan Zeng and Xinlong Hao and Shuai Yu and Dexun Li and Shuai Wang and Weinan Gan and Zhengying Liu and Yuanqing Yu and Zezhong Wang and Yuxian Wang and Wu Ning and Yutai Hou and Bin Wang and Chuhan Wu and Xinzhi Wang and Yong Liu and Yasheng Wang and Duyu Tang and Dandan Tu and Lifeng Shang and Xin Jiang and Ruiming Tang and Defu Lian and Qun Liu and Enhong Chen},
  booktitle = {The Thirteenth International Conference on Learning Representations},
  year = {2025}
}

@inproceedings{guo2024stabletoolbench,
  title = {{StableToolBench}: Towards Stable Large-Scale Benchmarking on Tool Learning of Large Language Models},
  author = {Guo, Zhicheng and Cheng, Sijie and Wang, Hao and Liang, Shihao and Qin, Yujia and Li, Peng and Liu, Zhiyuan and Sun, Maosong and Liu, Yang},
  editor = {Ku, Lun-Wei and Martins, Andre and Srikumar, Vivek},
  booktitle = {Findings of the Association for Computational Linguistics: ACL 2024},
  pages = {11143--11156},
  year = {2024},
  address = {Bangkok, Thailand},
  publisher = {Association for Computational Linguistics},
  doi = {10.18653/v1/2024.findings-acl.664}
}

@inproceedings{xie2024osworld,
  title = {{OSWorld}: Benchmarking Multimodal Agents for Open-Ended Tasks in Real Computer Environments},
  author = {Xie, Tianbao and Zhang, Danyang and Chen, Jixuan and Li, Xiaochuan and Zhao, Siheng and Cao, Ruisheng and Hua, Toh Jing and Cheng, Zhoujun and Shin, Dongchan and Lei, Fangyu and Liu, Yitao and Xu, Yiheng and Zhou, Shuyan and Savarese, Silvio and Xiong, Caiming and Zhong, Victor and Yu, Tao},
  editor = {A. Globerson and L. Mackey and D. Belgrave and A. Fan and U. Paquet and J. Tomczak and C. Zhang},
  booktitle = {Advances in Neural Information Processing Systems},
  volume = {37},
  pages = {52040--52094},
  year = {2024},
  publisher = {Curran Associates, Inc.},
  doi = {10.52202/079017-1650}
}

@inproceedings{yao2023react,
  title = {{ReAct}: Synergizing Reasoning and Acting in Language Models},
  author = {Shunyu Yao and Jeffrey Zhao and Dian Yu and Nan Du and Izhak Shafran and Karthik R Narasimhan and Yuan Cao},
  booktitle = {The Eleventh International Conference on Learning Representations},
  year = {2023}
}

@inproceedings{qin2023toolllm,
  title = {{ToolLLM}: Facilitating Large Language Models to Master 16000+ Real-world {API}s},
  author = {Yujia Qin and Shihao Liang and Yining Ye and Kunlun Zhu and Lan Yan and Yaxi Lu and Yankai Lin and Xin Cong and Xiangru Tang and Bill Qian and Sihan Zhao and Lauren Hong and Runchu Tian and Ruobing Xie and Jie Zhou and Mark Gerstein and Dahai Li and Zhiyuan Liu and Maosong Sun},
  booktitle = {The Twelfth International Conference on Learning Representations},
  year = {2024}
}

@article{li2025simia,
  title        = {Simulating Environments with Reasoning Models for Agent Training},
  author       = {Li, Yuetai and Inan, Huseyin A. and Yue, Xiang and Chen, Wei-Ning and Wutschitz, Lukas and Kulkarni, Janardhan and Poovendran, Radha and Sim, Robert and Rajmohan, Saravan},
  journal      = {arXiv preprint arXiv:2511.01824},
  year         = {2025}
}

@inproceedings{chen2026dreamgym,
  title = {Scaling Agent Learning via Experience Synthesis},
  author = {Zhaorun Chen and Zhuokai Zhao and Kai Zhang and Bo Liu and Qi Qi and Yifan Wu and Tarun Kalluri and Xuefei Cao and Yuanhao Xiong and Haibo Tong and Huaxiu Yao and Hengduo Li and Jiacheng Zhu and Xian Li and Dawn Song and Bo Li and Jason E Weston and Dat Huynh},
  booktitle = {The Fourteenth International Conference on Learning Representations},
  year = {2026}
}

@inproceedings{lu2024toolsandbox,
  title = {{ToolSandbox}: A Stateful, Conversational, Interactive Evaluation Benchmark for {LLM} Tool Use Capabilities},
  author = {Lu, Jiarui and Holleis, Thomas and Zhang, Yizhe and Aumayer, Bernhard and Nan, Feng and Bai, Haoping and Ma, Shuang and Ma, Shen and Li, Mengyu and Yin, Guoli and Wang, Zirui and Pang, Ruoming},
  editor = {Chiruzzo, Luis and Ritter, Alan and Wang, Lu},
  booktitle = {Findings of the Association for Computational Linguistics: NAACL 2025},
  pages = {1160--1183},
  year = {2025},
  address = {Albuquerque, New Mexico},
  publisher = {Association for Computational Linguistics},
  doi = {10.18653/v1/2025.findings-naacl.65}
}

@inproceedings{li2023apibank,
  title = {{API-Bank}: A Comprehensive Benchmark for Tool-Augmented {LLMs}},
  author = {Li, Minghao and Zhao, Yingxiu and Yu, Bowen and Song, Feifan and Li, Hangyu and Yu, Haiyang and Li, Zhoujun and Huang, Fei and Li, Yongbin},
  editor = {Bouamor, Houda and Pino, Juan and Bali, Kalika},
  booktitle = {Proceedings of the 2023 Conference on Empirical Methods in Natural Language Processing},
  pages = {3102--3116},
  year = {2023},
  address = {Singapore},
  publisher = {Association for Computational Linguistics},
  doi = {10.18653/v1/2023.emnlp-main.187}
}

@inproceedings{song2025envscaler,
    title = "{E}nv{S}caler: Scaling Tool-Interactive Environments for {LLM} Agent via Programmatic Synthesis",
    author = "Song, Xiaoshuai  and
      Chang, Haofei  and
      Dong, Guanting  and
      Zhu, Yutao  and
      Wen, Ji-Rong  and
      Dou, Zhicheng",
    editor = "Liakata, Maria  and
      Moreira, Viviane P.  and
      Zhang, Jiajun  and
      Jurgens, David",
    booktitle = "Findings of the {A}ssociation for {C}omputational {L}inguistics: {ACL} 2026",
    year = "2026",
    address = "San Diego, California, United States",
    publisher = "Association for Computational Linguistics",
    doi = "10.18653/v1/2026.findings-acl.407",
    pages = "8326--8357",
    ISBN = "979-8-89176-395-1"
}

@inproceedings{tu2025scaleenv,
title={{ScaleEnv}: Scaling Environment Synthesis from Scratch for Generalist Interactive Tool-Use Agent Training},
author={Dunwei Tu and Hongyan Hao and Hansi Yang and Yihao Chen and Yu Yang and Yueqing Sun and Xingchen Liu and Furao Shen and Qi GU and Hui Su and Xunliang Cai},
booktitle={Forty-third International Conference on Machine Learning},
year={2026}
}

@inproceedings{fang2025agentscaler,
    title = "Towards General Agentic Intelligence via Environment Scaling",
    author = "Fang, Runnan  and
      Cai, Shihao  and
      Li, Baixuan  and
      Wu, Jialong  and
      Li, Guangyu  and
      Yin, Wenbiao  and
      Wang, Xinyu  and
      Wang, Xiaobin  and
      Su, Liangcai  and
      Zhang, Zhen  and
      Wu, Shibin  and
      Tao, Zhengwei  and
      Jiang, Yong  and
      Xie, Pengjun  and
      Zhang, Ningyu  and
      Huang, Fei  and
      Zhang, Wentao  and
      Zhou, Jingren",
    editor = "Liakata, Maria  and
      Moreira, Viviane P.  and
      Zhang, Jiajun  and
      Jurgens, David",
    booktitle = "Findings of the {A}ssociation for {C}omputational {L}inguistics: {ACL} 2026",
    year = "2026",
    address = "San Diego, California, United States",
    publisher = "Association for Computational Linguistics",
    doi = "10.18653/v1/2026.findings-acl.872",
    pages = "17610--17621",
    ISBN = "979-8-89176-395-1"
}

@article{putta2024agentq,
  title = {Agent Q: Advanced Reasoning and Learning for Autonomous {AI} Agents},
  author = {Pranav Putta and Edmund Mills and Naman Garg and Sumeet Motwani and Chelsea Finn and Divyansh Garg and Rafael Rafailov},
  journal = {arXiv preprint arXiv:2408.07199},
  year = {2024}
}

@inproceedings{ross2011dagger,
  title = 	 {A Reduction of Imitation Learning and Structured Prediction to No-Regret Online Learning},
  author = 	 {Ross, Stephane and Gordon, Geoffrey and Bagnell, Drew},
  booktitle = 	 {Proceedings of the Fourteenth International Conference on Artificial Intelligence and Statistics},
  pages = 	 {627--635},
  year = 	 {2011},
  editor = 	 {Gordon, Geoffrey and Dunson, David and Dudík, Miroslav},
  volume = 	 {15},
  series = 	 {Proceedings of Machine Learning Research},
  address = 	 {Fort Lauderdale, FL, USA},
  publisher =    {PMLR}
}

@inproceedings{schick2023toolformer,
 author = {Schick, Timo and Dwivedi-Yu, Jane and Dessi, Roberto and Raileanu, Roberta and Lomeli, Maria and Hambro, Eric and Zettlemoyer, Luke and Cancedda, Nicola and Scialom, Thomas},
 booktitle = {Advances in Neural Information Processing Systems},
 doi = {10.52202/075280-2997},
 editor = {A. Oh and T. Naumann and A. Globerson and K. Saenko and M. Hardt and S. Levine},
 pages = {68539--68551},
 publisher = {Curran Associates, Inc.},
 title = {Toolformer: Language Models Can Teach Themselves to Use Tools},
 volume = {36},
 year = {2023}
}

@inproceedings{patil2023gorilla,
  title = {Gorilla: Large Language Model Connected with Massive {API}s},
  author = {Patil, Shishir G. and Zhang, Tianjun and Wang, Xin and Gonzalez, Joseph E.},
  editor = {A. Globerson and L. Mackey and D. Belgrave and A. Fan and U. Paquet and J. Tomczak and C. Zhang},
  booktitle = {Advances in Neural Information Processing Systems},
  volume = {37},
  pages = {126544--126565},
  year = {2024},
  publisher = {Curran Associates, Inc.},
  doi = {10.52202/079017-4020}
}

@inproceedings{liu2023agentbench,
title={{AgentBench}: Evaluating {LLM}s as Agents},
author={Xiao Liu and Hao Yu and Hanchen Zhang and Yifan Xu and Xuanyu Lei and Hanyu Lai and Yu Gu and Hangliang Ding and Kaiwen Men and Kejuan Yang and Shudan Zhang and Xiang Deng and Aohan Zeng and Zhengxiao Du and Chenhui Zhang and Sheng Shen and Tianjun Zhang and Yu Su and Huan Sun and Minlie Huang and Yuxiao Dong and Jie Tang},
booktitle={The Twelfth International Conference on Learning Representations},
year={2024}
}

@inproceedings{zeng2023agenttuning,
    title = "{A}gent{T}uning: Enabling Generalized Agent Abilities for {LLM}s",
    author = "Zeng, Aohan  and
      Liu, Mingdao  and
      Lu, Rui  and
      Wang, Bowen  and
      Liu, Xiao  and
      Dong, Yuxiao  and
      Tang, Jie",
    editor = "Ku, Lun-Wei  and
      Martins, Andre  and
      Srikumar, Vivek",
    booktitle = "Findings of the Association for Computational Linguistics: ACL 2024",
    year = "2024",
    address = "Bangkok, Thailand",
    publisher = "Association for Computational Linguistics",
    doi = "10.18653/v1/2024.findings-acl.181",
    pages = "3053--3077"
}

@article{chen2023fireact,
  title = {{FireAct}: Toward Language Agent Fine-tuning},
  author = {Baian Chen and Chang Shu and Ehsan Shareghi and Nigel Collier and Karthik Narasimhan and Shunyu Yao},
  journal = {arXiv preprint arXiv:2310.05915},
  year = {2023}
}

@inproceedings{song2024eto,
    title = "Trial and Error: Exploration-Based Trajectory Optimization of {LLM} Agents",
    author = "Song, Yifan  and
      Yin, Da  and
      Yue, Xiang  and
      Huang, Jie  and
      Li, Sujian  and
      Lin, Bill Yuchen",
    editor = "Ku, Lun-Wei  and
      Martins, Andre  and
      Srikumar, Vivek",
    booktitle = "Proceedings of the 62nd Annual Meeting of the Association for Computational Linguistics (Volume 1: Long Papers)",
    year = "2024",
    address = "Bangkok, Thailand",
    publisher = "Association for Computational Linguistics",
    doi = "10.18653/v1/2024.acl-long.409",
    pages = "7584--7600"
}

@inproceedings{xi2024agentgym,
  title = {{AgentGym}: Evaluating and Training Large Language Model-based Agents across Diverse Environments},
  author = {Xi, Zhiheng and Ding, Yiwen and Chen, Wenxiang and Hong, Boyang and Guo, Honglin and Wang, Junzhe and Guo, Xin and Yang, Dingwen and Liao, Chenyang and He, Wei and Gao, Songyang and Chen, Lu and Zheng, Rui and Zou, Yicheng and Gui, Tao and Zhang, Qi and Qiu, Xipeng and Huang, Xuanjing and Wu, Zuxuan and Jiang, Yu-Gang},
  editor = {Che, Wanxiang and Nabende, Joyce and Shutova, Ekaterina and Pilehvar, Mohammad Taher},
  booktitle = {Proceedings of the 63rd Annual Meeting of the Association for Computational Linguistics (Volume 1: Long Papers)},
  pages = {27914--27961},
  year = {2025},
  address = {Vienna, Austria},
  publisher = {Association for Computational Linguistics},
  doi = {10.18653/v1/2025.acl-long.1355}
}

@article{yang2025toolmind,
  title = {{ToolMind} Technical Report: A Large-Scale, Reasoning-Enhanced Tool-Use Dataset},
  author = {Chen Yang and Ran Le and Yun Xing and Zhenwei An and Zongchao Chen and Wayne Xin Zhao and Yang Song and Tao Zhang},
  journal = {arXiv preprint arXiv:2511.15718},
  year = {2025}
}

@article{zhao2025muarl,
  title        = {{MUA-RL}: Multi-turn User-interacting Agent Reinforcement Learning for Agentic Tool Use},
  author       = {Zhao, Weikang and Wang, Xili and Ma, Chengdi and Kong, Lingbin and Yang, Zhaohua and Tuo, Mingxiang and Shi, Xiaowei and Zhai, Yitao and Cai, Xunliang},
  journal      = {arXiv preprint arXiv:2508.18669},
  year         = {2025}
}

@inproceedings{zhang2026fissiongrpo,
  title = {Robust Tool Use via {Fission-GRPO}: Learning to Recover from Execution Errors},
  author = {Zhang, Zhiwei and Zhao, Fei and Wang, Rui and Wang, Zezhong and Liang, Bin and Wang, Jiakang and Hu, Yao and Cao, Shaosheng and Wong, Kam-Fai},
  editor = {Liakata, Maria and Moreira, Viviane P. and Zhang, Jiajun and Jurgens, David},
  booktitle = {Proceedings of the 64th Annual Meeting of the {A}ssociation for {C}omputational {L}inguistics (Volume 1: Long Papers)},
  pages = {40477--40491},
  year = {2026},
  address = {San Diego, California, United States},
  publisher = {Association for Computational Linguistics},
  doi = {10.18653/v1/2026.acl-long.1880}
}

@article{zhang2026visualseeker,
  title        = {Visual-Seeker: Towards Visual-Native Multimodal Agentic Search via Active Visual Reasoning},
  author       = {Zhang, Zhengbo and Miao, Changtao and Su, Jinbo and Zhou, Zhaowen and Zhang, Chunxia and Wang, Xukai and Liu, Ruiqi and Zheng, Kaiyuan and Cai, Jiansheng and Zhang, Bo and Li, Zhe and Xiang, Shiming and Yan, Ying},
  journal      = {arXiv preprint arXiv:2606.15231},
  year         = {2026}
}

@article{yang2026topocurate,
  title        = {TopoCurate: Modeling Interaction Topology for Tool-Use Agent Training},
  author       = {Yang, Jinluan and Liu, Yuxin and Chen, Zhengyu and Han, Chengcheng and Sun, Yueqing and Gu, Qi and Su, Hui and Cai, Xunliang and Wu, Fei and Kuang, Kun},
  journal      = {arXiv preprint arXiv:2603.01714},
  year         = {2026}
}

@article{parisi2022talm,
  title        = {{TALM}: Tool Augmented Language Models},
  author       = {Parisi, Aaron and Zhao, Yao and Fiedel, Noah},
  journal      = {arXiv preprint arXiv:2205.12255},
  year         = {2022}
}

@article{tang2023toolalpaca,
  title        = {{ToolAlpaca}: Generalized Tool Learning for Language Models with 3000 Simulated Cases},
  author       = {Tang, Qiaoyu and Deng, Ziliang and Lin, Hongyu and Han, Xianpei and Liang, Qiao and Cao, Boxi and Sun, Le},
  journal      = {arXiv preprint arXiv:2306.05301},
  year         = {2023}
}

@inproceedings{liu2024apigen,
  title = {{APIGen}: Automated Pipeline for Generating Verifiable and Diverse Function-Calling Datasets},
  author = {Liu, Zuxin and Hoang, Thai and Zhang, Jianguo and Zhu, Ming and Lan, Tian and Kokane, Shirley and Tan, Juntao and Yao, Weiran and Liu, Zhiwei and Feng, Yihao and Murthy, Rithesh and Yang, Liangwei and Savarese, Silvio and Niebles, Juan Carlos and Wang, Huan and Heinecke, Shelby and Xiong, Caiming},
  editor = {A. Globerson and L. Mackey and D. Belgrave and A. Fan and U. Paquet and J. Tomczak and C. Zhang},
  booktitle = {Advances in Neural Information Processing Systems},
  volume = {37},
  pages = {54463--54482},
  year = {2024},
  publisher = {Curran Associates, Inc.},
  doi = {10.52202/079017-1725}
}

@inproceedings{patil2025bfcl,
  title = {The Berkeley Function Calling Leaderboard ({BFCL}): From Tool Use to Agentic Evaluation of Large Language Models},
  author = {Patil, Shishir G and Mao, Huanzhi and Yan, Fanjia and Ji, Charlie Cheng-Jie and Suresh, Vishnu and Stoica, Ion and Gonzalez, Joseph E.},
  editor = {Singh, Aarti and Fazel, Maryam and Hsu, Daniel and Lacoste-Julien, Simon and Berkenkamp, Felix and Maharaj, Tegan and Wagstaff, Kiri and Zhu, Jerry},
  booktitle = {Proceedings of the 42nd International Conference on Machine Learning},
  volume = {267},
  pages = {48371--48392},
  year = {2025},
  publisher = {PMLR}
}

@inproceedings{zhuang2023toolqadatasetllmquestion,
 author = {Zhuang, Yuchen and Yu, Yue and Wang, Kuan and Sun, Haotian and Zhang, Chao},
 booktitle = {Advances in Neural Information Processing Systems},
 doi = {10.52202/075280-2180},
 editor = {A. Oh and T. Naumann and A. Globerson and K. Saenko and M. Hardt and S. Levine},
 pages = {50117--50143},
 publisher = {Curran Associates, Inc.},
 title = {{ToolQA}: A Dataset for {LLM} Question Answering with External Tools},
 volume = {36},
 year = {2023}
}

@inproceedings{huang2024metatoolbenchmarklargelanguage,
title={{MetaTool} Benchmark for Large Language Models: Deciding Whether to Use Tools and Which to Use},
author={Yue Huang and Jiawen Shi and Yuan Li and Chenrui Fan and Siyuan Wu and Qihui Zhang and Yixin Liu and Pan Zhou and Yao Wan and Neil Zhenqiang Gong and Lichao Sun},
booktitle={The Twelfth International Conference on Learning Representations},
year={2024}
}

@inproceedings{chen-etal-2024-eval,
    title = "{T-Eval}: Evaluating the Tool Utilization Capability of Large Language Models Step by Step",
    author = "Chen, Zehui  and
      Du, Weihua  and
      Zhang, Wenwei  and
      Liu, Kuikun  and
      Liu, Jiangning  and
      Zheng, Miao  and
      Zhuo, Jingming  and
      Zhang, Songyang  and
      Lin, Dahua  and
      Chen, Kai  and
      Zhao, Feng",
    editor = "Ku, Lun-Wei  and
      Martins, Andre  and
      Srikumar, Vivek",
    booktitle = "Proceedings of the 62nd Annual Meeting of the Association for Computational Linguistics (Volume 1: Long Papers)",
    month = aug,
    year = "2024",
    address = "Bangkok, Thailand",
    publisher = "Association for Computational Linguistics",
    doi = "10.18653/v1/2024.acl-long.515",
    pages = "9510--9529"
}

@inproceedings{basu-etal-2025-nestful,
    title = "{NESTFUL}: A Benchmark for Evaluating {LLM}s on Nested Sequences of {API} Calls",
    author = "Basu, Kinjal  and
      Abdelaziz, Ibrahim  and
      Kate, Kiran  and
      Agarwal, Mayank  and
      Crouse, Maxwell  and
      Rizk, Yara  and
      Bradford, Kelsey  and
      Munawar, Asim  and
      Kumaravel, Sadhana  and
      Goyal, Saurabh  and
      Wang, Xin  and
      Lastras, Luis A.  and
      Kapanipathi, Pavan",
    editor = "Christodoulopoulos, Christos  and
      Chakraborty, Tanmoy  and
      Rose, Carolyn  and
      Peng, Violet",
    booktitle = "Proceedings of the 2025 Conference on Empirical Methods in Natural Language Processing",
    month = nov,
    year = "2025",
    address = "Suzhou, China",
    publisher = "Association for Computational Linguistics",
    doi = "10.18653/v1/2025.emnlp-main.1702",
    pages = "33538--33547",
    ISBN = "979-8-89176-332-6"
}

@article{zhong2025complexfuncbenchexploringmultistepconstrained,
  title = {{ComplexFuncBench}: Exploring Multi-Step and Constrained Function Calling under Long-Context Scenario},
  author = {Lucen Zhong and Zhengxiao Du and Xiaohan Zhang and Haiyi Hu and Jie Tang},
  journal = {arXiv preprint arXiv:2501.10132},
  year = {2025}
}

@inproceedings{ye-etal-2025-toolhop,
    title = "{T}ool{H}op: A Query-Driven Benchmark for Evaluating Large Language Models in Multi-Hop Tool Use",
    author = "Ye, Junjie  and
      Du, Zhengyin  and
      Yao, Xuesong  and
      Lin, Weijian  and
      Xu, Yufei  and
      Chen, Zehui  and
      Wang, Zaiyuan  and
      Zhu, Sining  and
      Xi, Zhiheng  and
      Yuan, Siyu  and
      Gui, Tao  and
      Zhang, Qi  and
      Huang, Xuanjing  and
      Chen, Jiecao",
    editor = "Che, Wanxiang  and
      Nabende, Joyce  and
      Shutova, Ekaterina  and
      Pilehvar, Mohammad Taher",
    booktitle = "Proceedings of the 63rd Annual Meeting of the Association for Computational Linguistics (Volume 1: Long Papers)",
    month = jul,
    year = "2025",
    address = "Vienna, Austria",
    publisher = "Association for Computational Linguistics",
    doi = "10.18653/v1/2025.acl-long.150",
    pages = "2995--3021",
    ISBN = "979-8-89176-251-0"
}

@article{chen2025acebenchwinsmatchpoint,
  title = {{ACEBench}: Who Wins the Match Point in Tool Usage?},
  author = {Chen Chen and Xinlong Hao and Weiwen Liu and Xu Huang and Xingshan Zeng and Shuai Yu and Dexun Li and Shuai Wang and Weinan Gan and Yuefeng Huang and Wulong Liu and Xinzhi Wang and Defu Lian and Baoqun Yin and Yasheng Wang and Wu Liu},
  journal = {arXiv preprint arXiv:2501.12851},
  year = {2025}
}

@inproceedings{wang2024mintevaluatingllmsmultiturn,
title={{MINT}: Evaluating {LLM}s in Multi-turn Interaction with Tools and Language Feedback},
author={Xingyao Wang and Zihan Wang and Jiateng Liu and Yangyi Chen and Lifan Yuan and Hao Peng and Heng Ji},
booktitle={The Twelfth International Conference on Learning Representations},
year={2024}
}

@inproceedings{xu2025theagentcompanybenchmarkingllmagents,
  title = {{TheAgentCompany}: Benchmarking {LLM} Agents on Consequential Real World Tasks},
  author = {Frank F. Xu and Yufan Song and Boxuan Li and Yuxuan Tang and Kritanjali Jain and Mengxue Bao and Zora Zhiruo Wang and Xuhui Zhou and Zhitong Guo and Murong Cao and Mingyang Yang and Hao Yang Lu and Amaad Martin and Zhe Su and Leander Melroy Maben and Raj Mehta and Wayne Chi and Lawrence Keunho Jang and Yiqing Xie and Shuyan Zhou and Graham Neubig},
  booktitle = {Advances in Neural Information Processing Systems 38},
  pages = {10291--10321},
  year = {2025},
  publisher = {Neural Information Processing Systems Foundation, Inc. (NeurIPS)},
  doi = {10.52202/085713-0315}
}

@inproceedings{li2026tooldecathlonbenchmarkinglanguage,
title={The Tool Decathlon: Benchmarking Language Agents for Diverse, Realistic, and Long-Horizon Task Execution},
author={Junlong Li and Wenshuo Zhao and Jian Zhao and Weihao Zeng and Haoze Wu and Xiaochen Wang and Rui Ge and Yuxuan Cao and Yuzhen Huang and Wei Liu and Junteng Liu and Zhaochen Su and Yiyang Guo and Fan Zhou and Lueyang Zhang and Juan Michelini and Xingyao Wang and Xiang Yue and Shuyan Zhou and Graham Neubig and Junxian He},
booktitle={The Fourteenth International Conference on Learning Representations},
year={2026}
}

@inproceedings{wang2025mcpbenchbenchmarkingtoolusingllm,
title={{MCP}-Bench: Benchmarking Tool-Using {LLM} Agents with Complex Real-World Tasks via {MCP} Servers},
author={Zhenting Wang and Qi Chang and Hemani Patel and Shashank Biju and Cheng-En Wu and Quan Liu and Aolin Ding and Alireza Rezazadeh and Ankit Shah and Yujia Bao and Eugene Siow},
booktitle={The Fourteenth International Conference on Learning Representations},
year={2026}
}

@inproceedings{wu2025mcpmarkbenchmarkstresstestingrealistic,
title={{MCPM}ark: A Benchmark for Stress-Testing Realistic and Comprehensive {MCP} Use},
author={Zijian Wu and Xiangyan Liu and xinyuan zhang and Lingjun Chen and Fanqing Meng and Lingxiao Du and Yiran Zhao and Fanshi Zhang and Yaoqi Ye and Jiawei Wang and Zirui Wang and Jinjie Ni and Yufan Yang and Arvin Xu and Michael Qizhe Shieh},
booktitle={The Fourteenth International Conference on Learning Representations},
year={2026}
}

@inproceedings{yang2023gpt4toolsteachinglargelanguage,
  title = {{GPT}4Tools: Teaching Large Language Model to Use Tools via Self-instruction},
  author = {Yang, Rui and Song, Lin and Li, Yanwei and Zhao, Sijie and Ge, Yixiao and Li, Xiu and Shan, Ying},
  editor = {A. Oh and T. Naumann and A. Globerson and K. Saenko and M. Hardt and S. Levine},
  booktitle = {Advances in Neural Information Processing Systems},
  volume = {36},
  pages = {71995--72007},
  year = {2023},
  publisher = {Curran Associates, Inc.},
  doi = {10.52202/075280-3149}
}

@inproceedings{yin-etal-2024-agent,
    title = "Agent {Lumos}: Unified and Modular Training for Open-Source Language Agents",
    author = "Yin, Da  and
      Brahman, Faeze  and
      Ravichander, Abhilasha  and
      Chandu, Khyathi  and
      Chang, Kai-Wei  and
      Choi, Yejin  and
      Lin, Bill Yuchen",
    editor = "Ku, Lun-Wei  and
      Martins, Andre  and
      Srikumar, Vivek",
    booktitle = "Proceedings of the 62nd Annual Meeting of the Association for Computational Linguistics (Volume 1: Long Papers)",
    month = aug,
    year = "2024",
    address = "Bangkok, Thailand",
    publisher = "Association for Computational Linguistics",
    doi = "10.18653/v1/2024.acl-long.670",
    pages = "12380--12403"
}

@inproceedings{chen-etal-2024-agent,
    title = "Agent-{FLAN}: Designing Data and Methods of Effective Agent Tuning for Large Language Models",
    author = "Chen, Zehui  and
      Liu, Kuikun  and
      Wang, Qiuchen  and
      Zhang, Wenwei  and
      Liu, Jiangning  and
      Lin, Dahua  and
      Chen, Kai  and
      Zhao, Feng",
    editor = "Ku, Lun-Wei  and
      Martins, Andre  and
      Srikumar, Vivek",
    booktitle = "Findings of the Association for Computational Linguistics: ACL 2024",
    month = aug,
    year = "2024",
    address = "Bangkok, Thailand",
    publisher = "Association for Computational Linguistics",
    doi = "10.18653/v1/2024.findings-acl.557",
    pages = "9354--9366"
}

@inproceedings{zhang-etal-2025-xlam,
    title = "x{LAM}: A Family of Large Action Models to Empower {AI} Agent Systems",
    author = "Zhang, Jianguo  and
      Lan, Tian  and
      Zhu, Ming  and
      Liu, Zuxin  and
      Hoang, Thai  and
      Kokane, Shirley  and
      Yao, Weiran  and
      Tan, Juntao  and
      Liu, Zhiwei  and
      Feng, Yihao  and
      Niebles, Juan Carlos  and
      Heinecke, Shelby  and
      Wang, Huan  and
      Savarese, Silvio  and
      Xiong, Caiming",
    editor = "Chiruzzo, Luis  and
      Ritter, Alan  and
      Wang, Lu",
    booktitle = "Proceedings of the 2025 Conference of the Nations of the Americas Chapter of the Association for Computational Linguistics: Human Language Technologies (Volume 1: Long Papers)",
    month = apr,
    year = "2025",
    address = "Albuquerque, New Mexico",
    publisher = "Association for Computational Linguistics",
    doi = "10.18653/v1/2025.naacl-long.578",
    pages = "11583--11597",
    ISBN = "979-8-89176-189-6"
}

@article{lin2024hammerrobustfunctioncallingondevice,
  title = {Hammer: Robust Function-Calling for On-Device Language Models via Function Masking},
  author = {Qiqiang Lin and Muning Wen and Qiuying Peng and Guanyu Nie and Junwei Liao and Jun Wang and Xiaoyun Mo and Jiamu Zhou and Cheng Cheng and Yin Zhao and Jun Wang and Weinan Zhang},
  journal = {arXiv preprint arXiv:2410.04587},
  year = {2024}
}

@inproceedings{wang2025toolgenunifiedtoolretrieval,
title={{ToolGen}: Unified Tool Retrieval and Calling via Generation},
author={Renxi Wang and Xudong Han and Lei Ji and Shu Wang and Timothy Baldwin and Haonan Li},
booktitle={The Thirteenth International Conference on Learning Representations},
year={2025}
}

@inproceedings{song-etal-2024-agentbank,
    title = "{A}gent{B}ank: Towards Generalized {LLM} Agents via Fine-Tuning on 50000+ Interaction Trajectories",
    author = "Song, Yifan  and
      Xiong, Weimin  and
      Zhao, Xiutian  and
      Zhu, Dawei  and
      Wu, Wenhao  and
      Wang, Ke  and
      Li, Cheng  and
      Peng, Wei  and
      Li, Sujian",
    editor = "Al-Onaizan, Yaser  and
      Bansal, Mohit  and
      Chen, Yun-Nung",
    booktitle = "Findings of the Association for Computational Linguistics: EMNLP 2024",
    month = nov,
    year = "2024",
    address = "Miami, Florida, USA",
    publisher = "Association for Computational Linguistics",
    doi = "10.18653/v1/2024.findings-emnlp.116",
    pages = "2124--2141"
}

@article{wang2024learningfailureintegratingnegative,
  title = {Learning From Failure: Integrating Negative Examples when Fine-tuning Large Language Models as Agents},
  author = {Renxi Wang and Haonan Li and Xudong Han and Yixuan Zhang and Timothy Baldwin},
  journal = {arXiv preprint arXiv:2402.11651},
  year = {2024}
}

@inproceedings{xiong-etal-2024-watch,
    title = "Watch Every Step! {LLM} Agent Learning via Iterative Step-level Process Refinement",
    author = "Xiong, Weimin  and
      Song, Yifan  and
      Zhao, Xiutian  and
      Wu, Wenhao  and
      Wang, Xun  and
      Wang, Ke  and
      Li, Cheng  and
      Peng, Wei  and
      Li, Sujian",
    editor = "Al-Onaizan, Yaser  and
      Bansal, Mohit  and
      Chen, Yun-Nung",
    booktitle = "Proceedings of the 2024 Conference on Empirical Methods in Natural Language Processing",
    month = nov,
    year = "2024",
    address = "Miami, Florida, USA",
    publisher = "Association for Computational Linguistics",
    doi = "10.18653/v1/2024.emnlp-main.93",
    pages = "1556--1572"
}

@inproceedings{feng2025retoolreinforcementlearningstrategic,
title={{ReTool}: Reinforcement Learning for Strategic Tool Use in {LLM}s},
author={Jiazhan Feng and Shijue Huang and Xingwei Qu and Ge Zhang and Yujia Qin and Baoquan Zhong and Chengquan Jiang and Jinxin Chi and Wanjun Zhong},
booktitle={The Fourteenth International Conference on Learning Representations},
year={2026}
}

@article{li2025torlscalingtoolintegratedrl,
  title = {{ToRL}: Scaling Tool-Integrated {RL}},
  author = {Xuefeng Li and Haoyang Zou and Pengfei Liu},
  journal = {arXiv preprint arXiv:2503.23383},
  year = {2025}
}

@inproceedings{zhang2025nemotronresearchtooln1exploringtoolusinglanguage,
title={{Nemotron-Research-Tool-N1}: Exploring Tool-Using Language Models with Reinforced Reasoning},
author={Shaokun Zhang and Yi Dong and Jieyu Zhang and Jan Kautz and Bryan Catanzaro and Andrew Tao and Qingyun Wu and Zhiding Yu and Guilin Liu},
booktitle={The Fourteenth International Conference on Learning Representations},
year={2026}
}

@article{kimiteam2026kimik2openagentic,
  title = {Kimi {K2}: Open Agentic Intelligence},
  author = {Kimi Team and others},
  journal = {arXiv preprint arXiv:2507.20534},
  year = {2025}
}

@inproceedings{chen2025facilitatingmultiturnfunctioncalling,
title={Facilitating Multi-turn Function Calling for {LLM}s via Compositional Instruction Tuning},
author={Mingyang Chen and sunhaoze and Tianpeng Li and Fan Yang and Hao Liang and KeerLu and Bin CUI and Wentao Zhang and Zenan Zhou and weipeng chen},
booktitle={The Thirteenth International Conference on Learning Representations},
year={2025}
}

@inproceedings{yin-etal-2025-magnet,
    title = "Magnet: Multi-turn Tool-use Data Synthesis and Distillation via Graph Translation",
    author = "Yin, Fan  and
      Wang, Zifeng  and
      Hsu, I-Hung  and
      Yan, Jun  and
      Jiang, Ke  and
      Chen, Yanfei  and
      Gu, Jindong  and
      Le, Long  and
      Chang, Kai-Wei  and
      Lee, Chen-Yu  and
      Palangi, Hamid  and
      Pfister, Tomas",
    editor = "Che, Wanxiang  and
      Nabende, Joyce  and
      Shutova, Ekaterina  and
      Pilehvar, Mohammad Taher",
    booktitle = "Proceedings of the 63rd Annual Meeting of the Association for Computational Linguistics (Volume 1: Long Papers)",
    month = jul,
    year = "2025",
    address = "Vienna, Austria",
    publisher = "Association for Computational Linguistics",
    doi = "10.18653/v1/2025.acl-long.1566",
    pages = "32600--32616",
    ISBN = "979-8-89176-251-0"
}

@inproceedings{wang-etal-2025-toolflow,
    title = "{T}ool{F}low: Boosting {LLM} Tool-Calling Through Natural and Coherent Dialogue Synthesis",
    author = "Wang, Zezhong  and
      Zeng, Xingshan  and
      Liu, Weiwen  and
      Li, Liangyou  and
      Wang, Yasheng  and
      Shang, Lifeng  and
      Jiang, Xin  and
      Liu, Qun  and
      Wong, Kam-Fai",
    editor = "Chiruzzo, Luis  and
      Ritter, Alan  and
      Wang, Lu",
    booktitle = "Proceedings of the 2025 Conference of the Nations of the Americas Chapter of the Association for Computational Linguistics: Human Language Technologies (Volume 1: Long Papers)",
    month = apr,
    year = "2025",
    address = "Albuquerque, New Mexico",
    publisher = "Association for Computational Linguistics",
    doi = "10.18653/v1/2025.naacl-long.214",
    pages = "4246--4263",
    ISBN = "979-8-89176-189-6"
}

@inproceedings{xu2025agenttrekagenttrajectorysynthesis,
title={{AgentTrek}: Agent Trajectory Synthesis via Guiding Replay with Web Tutorials},
author={Yiheng Xu and Dunjie Lu and Zhennan Shen and Junli Wang and Zekun Wang and Yuchen Mao and Caiming Xiong and Tao Yu},
booktitle={The Thirteenth International Conference on Learning Representations},
year={2025}
}

@inproceedings{ou2024synatraturningindirectknowledge,
 author = {Ou, Tianyue and Xu, Frank F. and Madaan, Aman and Liu, Jiarui and Lo, Robert and Sridhar, Abishek and Sengupta, Sudipta and Roth, Dan and Neubig, Graham and Zhou, Shuyan},
 booktitle = {Advances in Neural Information Processing Systems},
 doi = {10.52202/079017-2908},
 editor = {A. Globerson and L. Mackey and D. Belgrave and A. Fan and U. Paquet and J. Tomczak and C. Zhang},
 pages = {91618--91652},
 publisher = {Curran Associates, Inc.},
 title = {Synatra: Turning Indirect Knowledge into Direct Demonstrations for Digital Agents at Scale},
 volume = {37},
 year = {2024}
}

@inproceedings{sun-etal-2025-os,
    title = "{OS-Genesis}: Automating {GUI} Agent Trajectory Construction via Reverse Task Synthesis",
    author = "Sun, Qiushi  and
      Cheng, Kanzhi  and
      Ding, Zichen  and
      Jin, Chuanyang  and
      Wang, Yian  and
      Xu, Fangzhi  and
      Wu, Zhenyu  and
      Jia, Chengyou  and
      Chen, Liheng  and
      Liu, Zhoumianze  and
      Kao, Ben  and
      Li, Guohao  and
      He, Junxian  and
      Qiao, Yu  and
      Wu, Zhiyong",
    editor = "Che, Wanxiang  and
      Nabende, Joyce  and
      Shutova, Ekaterina  and
      Pilehvar, Mohammad Taher",
    booktitle = "Proceedings of the 63rd Annual Meeting of the Association for Computational Linguistics (Volume 1: Long Papers)",
    month = jul,
    year = "2025",
    address = "Vienna, Austria",
    publisher = "Association for Computational Linguistics",
    doi = "10.18653/v1/2025.acl-long.277",
    pages = "5555--5579",
    ISBN = "979-8-89176-251-0"
}

@inproceedings{pahuja-etal-2025-explorer,
    title = "Explorer: Scaling Exploration-driven Web Trajectory Synthesis for Multimodal Web Agents",
    author = "Pahuja, Vardaan  and
      Lu, Yadong  and
      Rosset, Corby  and
      Gou, Boyu  and
      Mitra, Arindam  and
      Whitehead, Spencer  and
      Su, Yu  and
      Awadallah, Ahmed Hassan",
    editor = "Che, Wanxiang  and
      Nabende, Joyce  and
      Shutova, Ekaterina  and
      Pilehvar, Mohammad Taher",
    booktitle = "Findings of the Association for Computational Linguistics: ACL 2025",
    month = jul,
    year = "2025",
    address = "Vienna, Austria",
    publisher = "Association for Computational Linguistics",
    doi = "10.18653/v1/2025.findings-acl.326",
    pages = "6300--6323",
    ISBN = "979-8-89176-256-5"
}

@inproceedings{hu2025agentgenenhancingplanningabilities,
  title={{AgentGen}: Enhancing Planning Abilities for Large Language Model based Agent via Environment and Task Generation},
  DOI={10.1145/3690624.3709321},
  booktitle={Proceedings of the 31st ACM SIGKDD Conference on Knowledge Discovery and Data Mining V.1},
  publisher={ACM},
  author={Hu, Mengkang and Zhao, Pu and Xu, Can and Sun, Qingfeng and Lou, Jian-Guang and Lin, Qingwei and Luo, Ping and Rajmohan, Saravan},
  year={2025},
  pages={496--507}
}

@inproceedings{song2026agentdataprotocolunifying,
title={Agent Data Protocol: Unifying Datasets for Diverse, Effective Fine-tuning of {LLM} Agents},
author={Yueqi Song and Ketan Ramaneti and Zaid Sheikh and Ziru Chen and Boyu Gou and Tianbao Xie and Yiheng Xu and Danyang Zhang and Apurva Gandhi and Fan Yang and Joseph Liu and Tianyue Ou and Zhihao Yuan and Frank F. Xu and Shuyan Zhou and Xingyao Wang and Xiang Yue and Tao Yu and Huan Sun and Yu Su and Graham Neubig},
booktitle={The Fourteenth International Conference on Learning Representations},
year={2026}
}

@inproceedings{zhou2023limaalignment,
  title = {{LIMA}: Less Is More for Alignment},
  author = {Zhou, Chunting and Liu, Pengfei and Xu, Puxin and Iyer, Srinivasan and Sun, Jiao and Mao, Yuning and Ma, Xuezhe and Efrat, Avia and Yu, Ping and Yu, Lili and Zhang, Susan and Ghosh, Gargi and Lewis, Mike and Zettlemoyer, Luke and Levy, Omer},
  editor = {A. Oh and T. Naumann and A. Globerson and K. Saenko and M. Hardt and S. Levine},
  booktitle = {Advances in Neural Information Processing Systems},
  volume = {36},
  pages = {55006--55021},
  year = {2023},
  publisher = {Curran Associates, Inc.},
  doi = {10.52202/075280-2400}
}

@inproceedings{liu2024makesgooddataalignment,
title={What Makes Good Data for Alignment? A Comprehensive Study of Automatic Data Selection in Instruction Tuning},
author={Wei Liu and Weihao Zeng and Keqing He and Yong Jiang and Junxian He},
booktitle={The Twelfth International Conference on Learning Representations},
year={2024}
}

@inproceedings{chen2019learningcheating,
  title = 	 {Learning by Cheating},
  author =       {Chen, Dian and Zhou, Brady and Koltun, Vladlen and Kr\"ahenb\"uhl, Philipp},
  booktitle = 	 {Proceedings of the Conference on Robot Learning},
  pages = 	 {66--75},
  year = 	 {2020},
  editor = 	 {Kaelbling, Leslie Pack and Kragic, Danica and Sugiura, Komei},
  volume = 	 {100},
  series = 	 {Proceedings of Machine Learning Research},
  publisher =    {PMLR}
}

@article{qin2024toollearningfoundationmodels,
  title={Tool Learning with Foundation Models},
  volume={57},
  DOI={10.1145/3704435},
  number={4},
  journal={ACM Computing Surveys},
  publisher={Association for Computing Machinery (ACM)},
  author={Qin, Yujia and Hu, Shengding and Lin, Yankai and Chen, Weize and Ding, Ning and Cui, Ganqu and Zeng, Zheni and Zhou, Xuanhe and Huang, Yufei and Xiao, Chaojun and Han, Chi and Fung, Yi Ren and Su, Yusheng and Wang, Huadong and Qian, Cheng and Tian, Runchu and Zhu, Kunlun and Liang, Shihao and Shen, Xingyu and Xu, Bokai and Zhang, Zhen and Ye, Yining and Li, Bowen and Tang, Ziwei and Yi, Jing and Zhu, Yuzhang and Dai, Zhenning and Yan, Lan and Cong, Xin and Lu, Yaxi and Zhao, Weilin and Huang, Yuxiang and Yan, Junxi and Han, Xu and Sun, Xian and Li, Dahai and Phang, Jason and Yang, Cheng and Wu, Tongshuang and Ji, Heng and Li, Guoliang and Liu, Zhiyuan and Sun, Maosong},
  year={2024},
  pages={1--40}
}

@article{qu2024toollearninglargelanguage,
  title={Tool learning with large language models: a survey},
  volume={19},
  DOI={10.1007/s11704-024-40678-2},
  number={8},
  journal={Frontiers of Computer Science},
  publisher={Springer Science and Business Media LLC},
  author={Qu, Changle and Dai, Sunhao and Wei, Xiaochi and Cai, Hengyi and Wang, Shuaiqiang and Yin, Dawei and Xu, Jun and Wen, Ji-rong},
  year={2025},
  month=Jan
}

@inproceedings{lu2023chameleonplugandplaycompositionalreasoning,
 author = {Lu, Pan and Peng, Baolin and Cheng, Hao and Galley, Michel and Chang, Kai-Wei and Wu, Ying Nian and Zhu, Song-Chun and Gao, Jianfeng},
 booktitle = {Advances in Neural Information Processing Systems},
 doi = {10.52202/075280-1882},
 editor = {A. Oh and T. Naumann and A. Globerson and K. Saenko and M. Hardt and S. Levine},
 pages = {43447--43478},
 publisher = {Curran Associates, Inc.},
 title = {Chameleon: Plug-and-Play Compositional Reasoning with Large Language Models},
 volume = {36},
 year = {2023}
}

@inproceedings{shen2024taskbenchbenchmarkinglargelanguage,
 author = {Shen, Yongliang and Song, Kaitao and Tan, Xu and Zhang, Wenqi and Ren, Kan and Yuan, Siyu and Lu, Weiming and Li, Dongsheng and Zhuang, Yueting},
 booktitle = {Advances in Neural Information Processing Systems},
 doi = {10.52202/079017-0148},
 editor = {A. Globerson and L. Mackey and D. Belgrave and A. Fan and U. Paquet and J. Tomczak and C. Zhang},
 pages = {4540--4574},
 publisher = {Curran Associates, Inc.},
 title = {{TaskBench}: Benchmarking Large Language Models for Task Automation},
 volume = {37},
 year = {2024}
}

@article{wu2026terminaluniverseturningagenttrajectories,
  title = {{Terminal-Universe}: Turning Agent Trajectories into Scalable Terminal Environments},
  author = {Jie Wu and Zhenru Zhang and Beichen Zhang and Xuwu Wang and Yuhui Su and Mouxiang Chen and Peng Wang and Zhihai Wang and Que Shen and Hao Zhou and An Yang and Fei Huang and Yujiu Yang and Dayiheng Liu},
  journal = {arXiv preprint arXiv:2609.04148},
  year = {2026}
}

@inproceedings{jiang2026midtool,
title={{MidTool}: Mid-training Data Synthesis for Agentic Tool Use},
author={Fengqing Jiang and Yite Wang and Boyi Liu and Zhaoyang Wang and Canwen Xu and Zhewei Yao and Radha Poovendran and Yuxiong He},
booktitle={Third Conference on Language Modeling},
year={2026}
}

@inproceedings{zhang2025looptool,
  title = {{LoopTool}: Closing the Data--Training Loop for Robust {LLM} Tool Calls},
  author = {Zhang, Kangning and Liu, Weiwen and Jiao, Wenxiang and Du, Kounianhua and Lu, Yuan and Zhang, Weinan and Yu, Yong},
  editor = {Liakata, Maria and Moreira, Viviane P. and Zhang, Jiajun and Jurgens, David},
  booktitle = {Proceedings of the 64th Annual Meeting of the {A}ssociation for {C}omputational {L}inguistics (Volume 1: Long Papers)},
  pages = {21144--21164},
  year = {2026},
  address = {San Diego, California, United States},
  publisher = {Association for Computational Linguistics},
  doi = {10.18653/v1/2026.acl-long.968}
}

@article{ma2026patchestotrajectories,
  title = {From Patches to Trajectories: Privileged Process Supervision for Software-Engineering Agents},
  author = {Ma, Murong and Chen, Tianyu and Lin, Yun and Lu, Shuai and Zhu, Qinglin and Gong, Yeyun and Huang, Zhiyong and Cheng, Peng and Lu, Yan and Dong, Jin Song},
  journal = {arXiv preprint arXiv:2605.21996},
  year = {2026}
}
\bibliographystyle{plainnat}

\clearpage
\appendix
\section{Experimental Protocol}
\label{app:protocol}

\subsection{Benchmarks and Metrics}

We evaluate \method{} on Tau1, Tau2, and BFCL.
Tau1 contains 115 Retail and 50 Airline tasks, while Tau2 contains
40 Retail, 20 Airline, and 40 Telecom tasks.
For both Tau benchmarks, we report native task success within each domain
and macro-average equally across domains.

BFCL contains 1,718 examples spanning four Non-live categories, three Live
categories, and Multi-turn Base.
We use the same three-group Overall aggregation as in the main text,
equally averaging the Non-live mean, Live mean, and Multi-turn score.

\subsection{Training and Evaluation Populations}
\label{app:splits}

Training-side interaction construction follows the released data
partitions of each environment.
Tau1 Retail uses separate training and evaluation populations.
Tau1 Airline uses the public 50-task identity set for training-side
interaction construction and fresh-dialogue evaluation; we therefore
treat it as a seen-task, transductive setting rather than task-disjoint
generalization.

For Tau2 Retail, Airline, and Telecom, training-side construction uses
the official train split and evaluation uses the official test split.
BFCL follows its released training-side and evaluation memberships.
This provides example-level separation, although semantically related tool
behaviors may occur across the two populations.

\subsection{Comparison Protocol}
\label{app:evaluation}

Our primary comparison pairs each \method{} model with its corresponding
public Qwen3 initialization.
Within each scale, the two models are evaluated on the same benchmark
instances with the same agent interface, tool access, decoding
configuration, user simulator where applicable, and native scorer.
These paired evaluations isolate the effect of the post-training procedure
from differences in evaluation conditions.

We additionally report public agent-training checkpoints as external
references.
These models are evaluated on the same benchmark populations and scoring
rules, while retaining model-native prompting and tool-call interfaces
where necessary.
They therefore provide cross-method context rather than matched
interventions; full results and model-specific settings are given in
Appendix~\ref{app:public-checkpoints}.

Closed-source systems in Tables~\ref{tab:current-tau} and
\ref{tab:current-bfcl} are included only as contextual references under
their reported protocols.
\label{app:published-refs}
The BFCL results for GPT-5.5, Gemini-3.0-Pro, and ToolACE-2-8B follow
Table~3 of Gecko; we use the unrefined scores rather than
the additional GATS test-time computation.

\subsection{Training Configuration}
\label{app:optimization}

All \method{} models are initialized from their corresponding public Qwen3
checkpoints and trained for one epoch with full-parameter, assistant-token
causal cross-entropy.
We use a global batch size of eight rendered examples, bfloat16 precision,
AdamW with weight decay $0.01$, a cosine learning-rate schedule with ten
warmup steps, and a maximum gradient norm of $1.0$.
Training is non-thinking throughout, with no preference-learning,
process-reward, or reinforcement-learning objective.

Among optimization hyperparameters, only the learning rate varies with
model scale: $1.1\times10^{-5}$ for 1.7B,
$7\times10^{-6}$ for 4B and 8B, and
$3.7\times10^{-6}$ for 14B and 30B-A3B.
Each learner receives the scale-conditioned training experience defined by
the structural-support procedure in Appendix~\ref{app:selection}.

\Needspace{8\baselineskip}
\section{Interaction Synthesis and Data Composition}
\label{app:data}

\subsection{Training Interaction Corpus}
\label{app:data-statistics}

We treat each complete realized interaction as a semantic unit of training
experience.
It preserves the user request, the observations exposed through native
execution, and the grounded assistant decisions that can serve as learning
targets.
Keeping these elements within one interaction preserves the dependency
between earlier evidence and later decisions.

The 1.7B--14B learners draw from a common corpus of 34,772 certified
interactions spanning 2,515 training-side tasks.
The 30B-A3B learner draws from an expanded corpus of 35,819 interactions
spanning 3,602 tasks.
Both corpora use the same compositional representation and selection
principle; the expanded corpus broadens the set of available tasks and
interaction structures.

The selected views contain 11,583, 13,986, 14,819, 15,805, and 19,937
interactions for the 1.7B, 4B, 8B, 14B, and 30B-A3B learners,
respectively, corresponding to approximately 13.0M, 18.0M, 19.4M,
20.6M, and 23.9M supervised target tokens.
Task coverage remains fixed at 2,515 tasks for the 1.7B--14B views and
expands to 3,602 tasks for 30B-A3B.

The available experience is strongly skewed toward short, direct
interactions.
Within the common corpus, 40.9\% of interactions contain at most two
supervised assistant decisions, while only 10.0\% contain a grounded
decision after an authentic error observation.
Thus corpus size alone does not imply adequate exposure to recovery,
feedback-conditioned continuation, or other cross-step dependencies.
This imbalance motivates explicit control over compositional support.

\subsection{Compositional Structure Representation}
\label{app:composition-representation}

We represent an interaction by the dependencies that connect decisions
across time rather than by its sequence of native tool names.
A later action may depend on an earlier observation, user update,
environment response, or state transition.
Examples include reusing a retrieved identifier, revising a decision after
a rejection, or satisfying a prerequisite before a later state change.

To compare such dependencies across tasks and environments, we use a shared
capability vocabulary $\mathcal K$ that abstracts semantic operations,
observations, control-flow events, and interaction horizon.
Temporal order over these elements defines a compositional path.
The representation therefore captures how decisions depend on one another,
rather than the surface API through which they are expressed.

For example, retrieving an identifier and later using it in an update
instantiates the same information-dependent structure across different
tool APIs.
Likewise, a rejected action followed by information acquisition and a
successful continuation instantiates a recovery structure independently of
the backend that produced the rejection.

A crucial distinction is between \emph{intended} and \emph{realized}
structure.
A target composition can guide task construction, but it contributes
support only when the dependency is actually realized during native
execution.
A task intended to require recovery, for example, provides no recovery
support if the collected interaction never encounters and resolves the
relevant obstacle.
Structural support is therefore defined by the learner's realized
experience, not by the generator's specification.

\paragraph{Operational encoding.}
The released representation uses fixed feature families on the
policy-visible, decision-masked interaction.
Decision kinds distinguish text, tool calls, multiple calls, and handoffs.
Tool access is mapped to read, write, irreversible write, or handoff using
fixed name rules.
Object classes cover identity, commerce, travel, finance, connectivity,
communication, location, and time; argument roles cover identifiers,
amounts, temporal values, locations, contacts, status, quantity, and content.
Names are tokenized before these dictionary lookups.
Multiple matches form a sorted composite class, while unmatched names
receive a generic class.
Values contribute their types, such as string, number, list, or object,
rather than their literal contents.
Observation descriptors record return shape and, for object-valued
returns, the roles of their top-level fields.

Set-valued coverage additionally includes call and argument counts,
decision-horizon buckets, adjacent decision/access/object transitions, and
an error-recovery indicator.
For example, horizon buckets have upper boundaries $1,2,4,8,16,32$, with
a final overflow bucket.
The capability vocabulary comprises the instantiated labels from these
families, rather than a separate label for each native API.

\paragraph{A unique path for a complete interaction.}
Each supervised assistant decision contributes one descriptor containing
its decision kind, ordered call descriptors, subsequent consecutive tool
observations, and whether visible error feedback occurred since the
preceding assistant decision.
Argument-role/type pairs are sorted within a call; calls and observations
retain their recorded order.
A call also records a bucketed count of nonempty string arguments found
in preceding visible non-assistant evidence.
This is a representation feature applied after target eligibility has
been determined, not a replacement for grounding checks.

The canonical path is the ordered sequence of these descriptors, serialized
deterministically and assigned a stable key.
Thus a trajectory is not assigned to one manually chosen dependency:
its complete supervised decision sequence determines its path, while all
of its capability atoms remain available for coverage accounting.
Exact tool names, argument values, task identities, and benchmark labels
are absent from the serialized path.
Different native interactions share a path when their descriptors match;
changing decision order, observation shape, or recovery context can
produce a different path.

\paragraph{Mapping the recorded Retail interaction.}
Consider the training interaction in
Figure~\ref{fig:trajectory2gym-grounded-recovery} and
Table~\ref{tab:app-retail-recovery-mask}.
Its assistant decisions are indexed $0$--$19$.
Decision $1$ is the failed email lookup and remains context only.
The path therefore contains the descriptors of decisions
$0,2,\ldots,19$ in that order.
Decision $0$ is a text decision without error feedback; the recovery
decision following the failed lookup carries the after-error indicator.
The masked lookup is not counted as a supervised path step, but its
authentic observation changes the descriptor of the subsequent decision.

\Needspace{6\baselineskip}
The native operations in this workflow illustrate the fixed abstraction:
\begin{itemize}[leftmargin=1.4em,itemsep=0pt,topsep=2pt]
\item \texttt{get\_order\_details}: read access, commerce object.
\item \texttt{cancel\_pending\_order}: irreversible-write access, commerce object.
\item \texttt{modify\_pending\_order\_address}: write access,
commerce+location object.
\end{itemize}
An \texttt{order\_id} argument contributes identifier/string rather than
the particular order number.
The order status, returned identifiers, and user confirmations remain
in the underlying visible history used to certify actions.
The resulting path summarizes the ordered, eligible experience; it does
not replace that history or its execution evidence.

\subsection{Synthesis from Structural Deficits}
\label{app:synthesis}

Synthesis is directed toward compositional structures that remain weakly
represented in the current interaction corpus.
Given a target path, task construction produces a user request and an
initial environment state for which successful completion requires the
corresponding dependency.

Different dependencies induce different forms of construction.
Information-dependent tasks make a value needed later in the workflow
obtainable only through earlier visible evidence.
Recovery tasks introduce a resolvable obstacle whose cause and resolution
are exposed through native feedback.
State-dependent tasks require later operations to respect information or
state established earlier in the interaction.

The resulting task is then executed in the native environment by a
collection policy.
Observations, state transitions, user responses, errors, and terminal
outcomes are therefore consequences of execution rather than annotations
inserted by the generator.
Collection-side guidance may shape the attempted interaction, but it does
not become learner-visible evidence unless it is exposed through the
native interface.

After execution, the realized interaction is mapped back to the
compositional representation.
Only structures that actually occurred contribute support, and
decision-grounded supervision determines which realized assistant decisions
can become learning targets.
Subsequent synthesis is then directed toward structures that remain
under-supported.

Thus intended structure determines \emph{where to generate}, whereas
realized structure determines \emph{what experience was actually obtained}.

\section{Decision-Grounded Supervision}
\label{app:certification}

Decision grounding asks whether a realized assistant decision can be
reconstructed from the information available to the policy when that
decision is made.
This differs from execution validity: an action may be accepted by the
environment while still relying on information unavailable to the learner.
The grounding procedure therefore separates what happened in the collected
interaction from which decisions constitute reproducible supervision.

\subsection{Policy-Visible Evidence}

For a decision at step $t$, grounding is defined relative to the
policy-visible history $V(H_{<t})$.
Only information exposed through the native agent interface can justify the
decision.
User messages and assistant-visible tool observations may therefore provide
evidence, whereas collector-private state and private external operations
cannot directly support a target.

The relevant notion is information availability rather than textual
overlap.
A decision is grounded when the information needed to choose it is
recoverable from the preceding visible interaction.
For example, an identifier known only to the collection policy cannot
justify an assistant action even if the resulting call succeeds.
If the same identifier is subsequently returned through a visible tool
observation, later decisions may legitimately depend on it.

The same principle applies to dynamically available tools.
When tool discovery is part of the native interface, a tool can support a
supervised action only after its identity and usable interface have become
visible to the assistant.
More generally, arguments referring to users, orders, items, accounts, or
other environment entities must be supported by preceding visible evidence.

Private events may change the underlying environment without entering the
student history.
Their consequences become usable only when they are subsequently exposed
through an observation available to the assistant.
Grounding therefore preserves the distinction between latent environment
state and evidence available to the acting policy.
These checks implement specific requirements for visible provenance;
they do not prove that every admitted action is uniquely determined by
its visible history or verify every possible reason for choosing it.

\subsection{Decision and Workflow Eligibility}

A supervised decision must satisfy both execution and information
requirements.
Execution validity asks whether the action is admissible under the native
environment semantics; decision grounding asks whether its choice is
supported by the policy-visible history.
Neither condition subsumes the other.

This distinction is especially important for successful trajectories.
Trajectory-level success certifies an episode outcome, but does not imply
that every intermediate action is an appropriate imitation target.
A successful action that depends on collector-private information remains
ungrounded, while a grounded decision may occur only after several
observations or clarifications have accumulated.

For tasks with ordered state changes, local eligibility is supplemented by
workflow consistency.
Required operations must appear in the realized interaction in an order
compatible with the task.
Information gathering, user clarification, and recovery may occur between
those operations, but a rejected or masked action cannot satisfy a required
workflow step simply because the trajectory eventually succeeds.

The resulting supervision therefore distinguishes three levels:
\emph{trajectory success}, which concerns completion of the interaction;
\emph{execution validity}, which concerns whether an individual action is
accepted by the environment; and \emph{decision grounding}, which concerns
whether that action is justified by information available when it is taken.
Only decisions satisfying the relevant levels contribute imitation loss.

Native environments occasionally provide action-level validation without a
corresponding observation.
Such validation can authorize the action as a target, but does not create an
unobserved action--observation transition.

\subsection{Failure as Evidence for Recovery}

An action need not be a valid imitation target to contribute useful
information to the interaction.
This is particularly important for recovery.

When an assistant action fails, the failed action itself receives no
imitation loss.
If the action and its authentic error response are visible through the
native interface, however, they remain part of the policy-visible history.
A later recovery decision can then be supervised from the richer context
created by that feedback.
The learner therefore observes the consequence of the failed attempt
without being trained to reproduce the failure.

Repeated actions are treated according to the same principle.
A repeated read after a state transition may reveal new information and
support a new decision, even when its serialized action matches an earlier
call.
Likewise, a repeated operation may be required by the task after the
environment state has changed.
Surface repetition alone is therefore insufficient to identify redundancy.

By contrast, a repeated action that contributes no new evidence, state
change, or required operation provides no additional compositional support.
The distinction is semantic: what matters is whether the repeated action
changes the information or state on which subsequent decisions depend.

This treatment allows one realized interaction to contain both masked
failures and valid supervision for their recovery.
Figure~\ref{fig:trajectory2gym-grounded-recovery} provides a concrete
example in Appendix~\ref{app:native-interaction-cases}.

\section{Structural-Support Learning}
\label{app:selection}

Grounding determines which realized decisions can provide valid
supervision, but it does not determine how certified experience should be
distributed across interaction structures.
Structural-support learning addresses two remaining questions:
which interactions should be retained for a learner, and how much aggregate
optimization weight each retained interaction should receive.

\subsection{Structural Coverage and Support Depth}

We distinguish whether a compositional structure is represented from how
much repeated support it receives.
Let $\mathcal P(\tau)$ denote the compositional structures realized by
interaction $\tau$.
For repeated-support accounting, each interaction is associated with one
canonical path $p(\tau)\in\mathcal P(\tau)$.
This assignment does not imply that the interaction contains only one
useful dependency; its remaining capability content continues to contribute
to structural coverage.

For a selected set $\mathcal S$, the support of path $p$ is
\begin{equation}
    n_{\mathcal S}(p)
    =
    \sum_{\tau\in\mathcal S}
    \mathbf{1}[p(\tau)=p].
    \label{eq:app-path-support}
\end{equation}

For learner scale $s$, let $d_s$ denote the desired support depth.
The remaining deficit of path $p$ is
\begin{equation}
    \delta_s(p;\mathcal S)
    =
    \bigl[d_s-n_{\mathcal S}(p)\bigr]_+.
    \label{eq:app-support-deficit}
\end{equation}

For an unselected interaction $\tau$, its contribution to repeated
structural support is
\begin{equation}
\begin{aligned}
    \Delta_s(\tau\mid\mathcal S)
    &=
    \delta_s(p(\tau);\mathcal S)
    -
    \delta_s\!\left(
        p(\tau);
        \mathcal S\cup\{\tau\}
    \right) \\
    &=
    \mathbf{1}
    \bigl[
        n_{\mathcal S}(p(\tau))<d_s
    \bigr].
\end{aligned}
\label{eq:marginal-structural-support}
\end{equation}

Thus $d_s=1$ corresponds to path coverage, while larger values allocate
repeated experience to structures that are already represented.
For the scaling study, we use $d_s=1,2,3,4,$ and $8$ for the 1.7B, 4B,
8B, 14B, and 30B-A3B learners, respectively.

The depth schedule specifies how training experience is allocated; it is
not a monotonic claim about the amount of support that maximizes downstream
performance.
The fixed-8B sensitivity study in Table~\ref{tab:support-depth} confirms
this distinction: increasing depth adds interactions and target tokens, but
the three-benchmark mean does not improve monotonically.

\subsection{Experience Allocation}

Selection balances two complementary objectives:
preserving breadth over tasks and compositional structures, and providing
repeated support for structures that remain under-represented.

We begin from a fixed certified core shared across training views.
Selection then proceeds hierarchically.
Task coverage preserves completing demonstrations across the available
training tasks; structural coverage adds interactions that expose previously
uncovered capability atoms and canonical paths; finally, support expansion
allocates additional interactions toward paths with positive deficit
$\delta_s$.
This ordering prevents repeated exposure to frequent structures from
displacing broad task and capability coverage.

Within each stage, interactions that provide comparable structural support
are distinguished by secondary criteria.
Task-level selection favors concise, completing demonstrations, while
structural selection favors interactions that add visible capability
content with limited redundant supervision.
Recovery and branching are retained when they contribute additional
interaction structure rather than merely increasing trajectory length.

The resulting procedure is deliberately independent of downstream
evaluation outcomes.
Benchmark identity, evaluation score, teacher identity, and test-task
identity are not used as selection features.
Likewise, negative log-likelihood under the public initialization does not
determine which interactions are retained.
Training-set composition therefore emerges from task breadth,
compositional coverage, and structural-support deficits rather than from
benchmark-specific quotas or difficulty scores.

Algorithm~\ref{alg:structural-support-allocation} makes this ordering explicit.
For candidate $\tau$, write $T_\tau>0$ for its certified target-token count,
$u(\tau)$ for its training task, and $\mathcal K_\tau$ for its capability
atoms.
The task-anchor key $A(\tau)$ orders candidates by fewer exact repeated
calls, fewer total calls, fewer certified target tokens, more supervised
text decisions, and finally ascending trace key.
Structural comparisons maximize the lexicographic key
$Q(\tau)=(-\rho_\tau,-\lambda_\tau,r_\tau,b_\tau,e_\tau,-T_\tau,h_\tau)$:
lower unsupported-reference risk and redundancy precede higher recovery,
branching, and cost-efficiency features, followed by token count and a
deterministic trace key.
These features are computed before allocation; $h_\tau$ is not a quality
score.
Recovery, branching, and cost efficiency use the corpus-normalized feature
values supplied to the selector.

\begin{algorithm}[!htbp]
\caption{Structural-support allocation within a fixed eligible corpus}
\label{alg:structural-support-allocation}
\small
\begin{algorithmic}[1]
\Require Eligible interactions $\mathcal C$, retained core $\mathcal C_0$,
ordered depths $\{d_s\}$, per-task expansion cap $K$
\Ensure Selected view $\mathcal S_s$ for each requested depth
\State $\mathcal S\gets\mathcal C_0$; initialize task/path counts and covered atoms
\For{each uncovered training task $u$, in canonical order}
    \State $\operatorname{admit}(\arg\min_{\tau\in\mathcal C:\,u(\tau)=u} A(\tau))$
\EndFor
\While{some candidate capability atom is uncovered}
    \State $g(\tau)\gets|\mathcal K_\tau\setminus\bigcup_{\sigma\in\mathcal S}\mathcal K_\sigma|$
    \State $\mathcal F\gets\{\tau\in\mathcal C\setminus\mathcal S:
        c_{\mathcal S}(u(\tau))<K,\ g(\tau)>0\}$
    \If{$\mathcal F=\varnothing$}
        \State \textbf{stop}: coverage is infeasible under the supplied cap
    \EndIf
    \State $\operatorname{admit}
        (\arg\max_{\tau\in\mathcal F}(g(\tau)/T_\tau,g(\tau),Q(\tau)))$
\EndWhile
\For{each requested depth $d_s$, in increasing order}
    \For{$k=1,\ldots,d_s$}
        \State Build a queue with the highest-$Q$ eligible candidate of each path
        \Statex \hspace{\algorithmicindent}\hspace{\algorithmicindent}
            whose current support is below $k$
        \For{each queued $\tau$, in descending $(Q(\tau),p(\tau))$ order}
            \If{$n_{\mathcal S}(p(\tau))<k$ and $c_{\mathcal S}(u(\tau))<K$}
                \State $\operatorname{admit}(\tau)$
            \EndIf
        \EndFor
    \EndFor
    \State $\mathcal S_s\gets\mathcal S$
\EndFor
\end{algorithmic}
\end{algorithm}

Here $\operatorname{admit}$ adds one unselected interaction and updates
all counts and covered atoms; $c_{\mathcal S}(u)$ denotes the task count
in the algorithm.
The retained core takes precedence over the expansion cap.
Candidates queued for support expansion must be unselected and below
the task cap when the queue is built, and the cap is checked again at
admission because earlier choices can consume the remaining capacity.
Coverage closes before repeated support is added.
Within each support level, the queue is built once, giving each eligible
under-supported path an opportunity before advancing to the next level.

Allocation stops after the requested levels have been processed.
Depth is a support target, not a hard upper bound: retained interactions
are never removed to reduce an over-supported path, and candidate
availability or the task cap can leave another path below the target.
For successive depths on the same corpus, previously selected
interactions remain selected.
The emitted view is ordered by task, path, and trace key; rendering and
parent-normalized training are applied afterward.

\paragraph{Worked support allocation.}
For illustration, suppose core retention and breadth coverage have produced
four selected interactions with path counts
$(n_{\mathcal S}(p_A),n_{\mathcal S}(p_B),n_{\mathcal S}(p_C))=(3,1,0)$.
Assume that the task and atom requirements are already satisfied and that
additional candidates exist for $p_B$ and $p_C$.
At depth two, the first support level adds one $p_C$ interaction, producing
$(3,1,1)$.
The second level can then add one interaction for each of $p_B$ and $p_C$,
producing $(3,2,2)$ when the task cap does not bind.
The three existing $p_A$ interactions remain selected even though their
support already exceeds the requested depth.

Now impose $K=3$ and let two selected $p_A$ interactions and the selected
$p_B$ interaction belong to task $u_1$.
Its selected-task count is already three.
Let the remaining $p_A$ interaction and all available $p_C$ candidates
belong to task $u_2$.
If every additional $p_B$ candidate belongs to $u_1$, those candidates
are ineligible for expansion, regardless of their quality ordering.
The two support levels can still add two $p_C$ interactions to $u_2$,
yielding $(3,1,2)$.
The remaining deficit on $p_B$ reflects a binding feasibility constraint,
not a missing update to the support counter.

\Needspace{5\baselineskip}
This example separates the roles of the three controls.
Core retention preserves selected experience, support depth determines
which paths are eligible for additional repetition, and the task cap
limits concentration within one training task.
Quality priorities order candidates within these constraints.
A larger depth neither removes retained experience nor overrides a
saturated task cap, and increasing it need not increase every path's
support.
Equation~\eqref{eq:marginal-structural-support} characterizes one component
of the allocation rule; it does not imply global optimality of the complete
multi-stage selector.
Appendix~\ref{app:theory-selection} analyzes the idealized deficit stage
under a pure cardinality budget.

\subsection{Parent-Normalized Learning Objective}
\label{app:parent-normalization}

Structural selection operates on complete semantic interactions, whereas a
single interaction may render into multiple native training examples.
A conventional row-averaged objective would therefore assign greater
optimization mass to interactions that happen to produce more rendered
examples.
We instead normalize the objective at the semantic-interaction level.

Suppose selected interaction $i$ renders into examples
$j\in\mathcal J_i$.
Let $m_{ijk}\in\{0,1\}$ denote the causal supervision mask for token $k$
in rendered example $j$, let $n_{ij}=\sum_k m_{ijk}>0$ be its number of
supervised tokens, and let $N_i=\sum_{j\in\mathcal J_i}n_{ij}$ be the total
number of supervised tokens associated with interaction $i$.
The mean certified-token loss of rendered example $j$ is
\begin{equation}
    \ell_{ij}(\theta)
    =
    -\frac{1}{n_{ij}}
    \sum_k
    m_{ijk}
    \log
    \pi_\theta
    \bigl(
        y_{ijk}
        \mid
        \widetilde h_{ij,<k}
    \bigr).
    \label{eq:app-child-loss}
\end{equation}

If $P_s$ semantic interactions are selected for learner scale $s$, the
parent-normalized empirical objective is
\begin{equation}
    \mathcal L_{\mathrm{parent}}(\theta)
    =
    \frac{1}{P_s}
    \sum_i
    \sum_{j\in\mathcal J_i}
    \frac{n_{ij}}{N_i}
    \ell_{ij}(\theta).
    \label{eq:app-parent-risk}
\end{equation}

The inner weights sum to one for every semantic interaction.
Consequently, each selected interaction receives the same aggregate
coefficient $1/P_s$, while its certified target occurrences share that
mass internally.
This separates the intended importance of an interaction from the number
of rendered examples used to represent it.

If $M_s=\sum_i|\mathcal J_i|$ is the number of rendered examples,
the same objective can be implemented with per-example coefficient
\begin{equation}
    \omega_{ij}
    =
    \frac{M_s}{P_s}
    \frac{n_{ij}}{N_i}.
    \label{eq:app-child-weight}
\end{equation}
Indeed,
\begin{equation}
    \frac{1}{M_s}
    \sum_i
    \sum_{j\in\mathcal J_i}
    \omega_{ij}\ell_{ij}(\theta)
    =
    \mathcal L_{\mathrm{parent}}(\theta).
    \label{eq:app-normalization-equivalence}
\end{equation}

Hence, for a fixed set of certified target occurrences and unchanged
conditioning histories, repartitioning one semantic interaction into a
different number of rendered examples does not change its intended
aggregate weight.
This is the rendering-invariance property analyzed formally in
Appendix~\ref{app:theory-parent}.
Finite optimization trajectories may still depend on sample ordering and
batch construction.

\section{Theoretical Notes on the Construction}
\label{app:theory}

This appendix records elementary properties of the three stages of
Section~\ref{sec:method}, in the notation of
Section~\ref{sec:formulation}. They are not generalization guarantees;
they make precise what each stage can and cannot do, and they are the
basis for the justifications given in the main text.

\subsection{Grounding and Hidden Information}
\label{app:theory-grounding}

A collected decision is a triple $(c, z, a)$: the student-visible context
$c = c_t$ of Eq.~\eqref{eq:tool-use-policy}, the collector-private
information $z$ (teacher hints, private events, identifiers known only to
the collector), and the assistant action $a$. Let $(C, Z, A)$ denote a
draw from the collection distribution restricted to any set of decisions,
and let $H$ and $I$ denote conditional entropy and mutual information.
A student is any policy $\pi_\theta(a \mid c)$; it never observes $z$.

\begin{proposition}[Grounding floor]
\label{prop:grounding}
For every student policy,
\begin{equation}
  \mathbb E\bigl[-\log \pi_\theta(A \mid C)\bigr]
  \;\ge\; H(A \mid C)
  \;=\; H(A \mid C, Z) + I(A; Z \mid C),
  \label{eq:grounding-floor}
\end{equation}
with equality in the inequality if and only if
$\pi_\theta(\cdot \mid c) = p(\cdot \mid c)$ for almost every $c$. If
the collector acts deterministically given $(c, z)$, the floor equals
$I(A; Z \mid C)$, the information about the action that is carried by
private state and unavailable to the student.
\end{proposition}
\begin{proof}
$\mathbb E[-\log \pi_\theta(A\mid C)] - H(A\mid C)
= \mathbb E_C\bigl[\mathrm{KL}\bigl(p(\cdot\mid C)\,\|\,\pi_\theta(\cdot\mid C)\bigr)\bigr]
\ge 0$ by Gibbs' inequality, with equality if and only if the two
conditionals agree almost surely. The decomposition is the chain rule
$H(A\mid C) = H(A\mid C,Z) + I(A;Z\mid C)$, and $H(A\mid C,Z)=0$ when
$A$ is a function of $(C,Z)$.
\end{proof}

\begin{remark}
The grounding check excludes targets whose identifier-like arguments
lack preceding visible support. Exclusion by this operational rule does
not by itself imply $I(A;Z\mid C)>0$ in the collection distribution.
When an action retains dependence on collector-private information after
conditioning on visible context, Proposition~\ref{prop:grounding}
identifies the additional optimal log loss faced by a student restricted
to $C$ relative to one observing $(C,Z)$. This motivates excluding
unsupported targets; it does not assert that every excluded target has a
strictly positive information term or that every admitted action is
uniquely determined by $C$.
\end{remark}

\begin{remark}[Certification is a nested filter]
Since $m_t = R_\tau\,\mathbf 1[q_\tau \preceq A_\tau]\,e_t\,v_t$ is a
product of indicators, the certified decisions are exactly those passing
all four checks, a decision failing any one check is excluded regardless
of the others, and the eligible pool $\{\tau : \sum_t m_t > 0\}$ is a
subset of the success-filtered pool $\{\tau : R_\tau = 1\}$.
Certification therefore only ever removes targets relative to outcome
filtering; Proposition~\ref{prop:grounding} identifies what is removed.
\end{remark}

\subsection{Deficit-Driven Selection Maximizes Capped Support}
\label{app:theory-selection}

For a candidate pool $\mathcal C$, a selected set
$\mathcal S \subseteq \mathcal C$, and a cap $d \ge 1$, define the capped
support
\begin{equation}
  F_d(\mathcal S) = \sum_{p \in \mathcal P} \min\bigl(n_{\mathcal S}(p),\, d\bigr),
  \qquad
  \sum_{p \in \mathcal P} \delta_d(p;\mathcal S) = d\,|\mathcal P| - F_d(\mathcal S),
  \label{eq:capped-support}
\end{equation}
where $\mathcal P$ is the set of feasible canonical paths and
$\delta_d(p;\mathcal S) = [d - n_{\mathcal S}(p)]_+$ as in
Eq.~\eqref{eq:support-deficit}. Reducing the total deficit is therefore
the same as increasing $F_d$. Let $N_p$ be the number of pool candidates
whose canonical path is $p$.

\begin{proposition}[Exactness of the deficit stage]
\label{prop:selection}
\begin{enumerate}[leftmargin=*,itemsep=1pt,topsep=1pt]
\item $F_d$ is monotone and submodular on subsets of $\mathcal C$, with
marginal gain $\Delta_d(\tau\mid\mathcal S) = \mathbf 1[n_{\mathcal S}(p(\tau)) < d]$,
which is Eq.~\eqref{eq:marginal-structural-support}.
\item Starting from any fixed set $\mathcal S_0$ and adding at most $B$
candidates, the greedy rule that repeatedly adds a candidate with
positive marginal gain attains
\begin{equation}
  \max_{\mathcal S \supseteq \mathcal S_0,\ |\mathcal S\setminus\mathcal S_0| \le B}
  F_d(\mathcal S)
  = \min\Bigl(F_d(\mathcal S_0) + B,\ \sum_{p \in \mathcal P} \min(N_p, d)\Bigr).
  \label{eq:greedy-optimum}
\end{equation}
\item Without a cap ($d = \infty$), $F_\infty(\mathcal S) = |\mathcal S|$
is modular: every budget-feasible selection is optimal and the objective
expresses no preference among paths. The cap is what encodes breadth,
and $d = 1$ counts covered paths.
\end{enumerate}
\end{proposition}
\begin{proof}
(1) Each candidate contributes to exactly one path, so
$n_{\mathcal S}(p)$ is modular in $\mathcal S$, and $n \mapsto \min(n, d)$
is nondecreasing and concave; a sum of nondecreasing concave functions
of modular functions is monotone submodular. Adding $\tau$ raises
$\min(n_{\mathcal S}(p(\tau)), d)$ by one exactly when
$n_{\mathcal S}(p(\tau)) < d$.
(2) Upper bound: each added candidate raises $F_d$ by at most one, and
$F_d(\mathcal S) \le \sum_p \min(N_p, d)$ for every
$\mathcal S \subseteq \mathcal C$. Achievability: whenever
$F_d(\mathcal S) < \sum_p \min(N_p, d)$ there is a path with
$n_{\mathcal S}(p) < \min(N_p, d)$, hence an unselected candidate of that
path with marginal gain one; greedy adds one per step, so $F_d$ grows by
exactly one per step until the budget or the bound is reached.
(3) Immediate from $\min(n,\infty) = n$.
\end{proof}

\begin{remark}
The implemented selection has several stages (a fixed core, one
completing exemplar per task, atom and path coverage, then deficit
reduction) and breaks ties by grounding risk, recovery, branching,
redundancy, and cost (Appendix~\ref{app:selection}).
Proposition~\ref{prop:selection} concerns the deficit stage with
$\mathcal S_0$ the output of the earlier stages: its greedy order is exact
for the capped objective, and tie-breaking only chooses among candidates
with equal marginal gain. The coverage stage is the $d = 1$ case.
\end{remark}

\subsection{Random Surplus Selection Misses Rare Paths}
\label{app:theory-random}

\begin{proposition}[Miss probability under uniform selection]
\label{prop:random}
Suppose $B$ surplus candidates are drawn uniformly without replacement
from a pool of $N$ candidates, and path $p$ has $N_p$ candidates in that
pool. Then
\begin{equation}
  \Pr\bigl[n_{\mathcal S}(p) = 0\bigr]
  = \frac{\binom{N - N_p}{B}}{\binom{N}{B}}
  = \prod_{j=0}^{N_p - 1} \frac{N - B - j}{N - j}
  \;\le\; \Bigl(1 - \frac{B}{N}\Bigr)^{N_p},
  \label{eq:miss-probability}
\end{equation}
with equality for $N_p = 1$: a path represented once is missed with
probability exactly $1 - B/N$. The expected number of covered paths is
$\sum_p \bigl(1 - \Pr[n_{\mathcal S}(p) = 0]\bigr)$.
\end{proposition}
\begin{proof}
$n_{\mathcal S}(p)$ is hypergeometric; the product form counts the
$B$-subsets that avoid all $N_p$ candidates of $p$, and each factor is at
most $(N - B)/N$ because $(N-B-j)/(N-j)$ is nonincreasing in $j$.
\end{proof}

\begin{remark}
Coverage-first selection covers every feasible path whenever the budget
allows (Proposition~\ref{prop:selection} with $d = 1$), whereas uniform
selection misses rare paths with a probability that decays only
geometrically in $N_p$. Taking the overall selection ratio of the
fixed-8B comparison as an illustration
($B/N \approx 14{,}819/34{,}772 \approx 0.43$), a path represented once is
missed with probability about $0.57$ and one represented three times with
probability about $0.19$. This is the mechanism behind the 942 canonical
paths that random surplus selection left uncovered at the same budget
(Figure~\ref{fig:fixed8b-diagnostics}); the actual control draws the
surplus after fixed anchors, so the numbers are illustrative.
\end{remark}

\subsection{Parent Normalization}
\label{app:theory-parent}

Let parent $i$ render into children $j \in \mathcal J_i$ with certified
target counts $n_{ij}$ and $N_i = \sum_j n_{ij}$, and let $P_s$ and
$M_s = \sum_i |\mathcal J_i|$ be the numbers of parents and children.
Write $\ell_{ij}$ for the child-mean loss of
Eq.~\eqref{eq:app-parent-risk}, and let
$\mathcal L_{\mathrm{row}} = \frac{1}{M_s}\sum_{i,j} \ell_{ij}$ denote
per-example (row) averaging.

\begin{proposition}[Training mass under the two objectives]
\label{prop:parent}
\begin{enumerate}[leftmargin=*,itemsep=1pt,topsep=1pt]
\item Under $\mathcal L_{\mathrm{row}}$, parent $i$ receives total
coefficient $|\mathcal J_i| / M_s$, so the coefficients of two parents are
in the ratio of their child counts; under $\mathcal L_{\mathrm{parent}}$
every parent receives $1/P_s$.
\item Under $\mathcal L_{\mathrm{parent}}$ every certified target
occurrence of parent $i$ has weight $1/(P_s N_i)$; under
$\mathcal L_{\mathrm{row}}$ an occurrence in child $j$ has weight
$1/(M_s n_{ij})$, so occurrences in short children are up-weighted
relative to those in long children.
\item For a fixed multiset of target occurrences with fixed conditioning
prefixes, $\mathcal L_{\mathrm{parent}}$ is invariant to how the
occurrences are partitioned into children, whereas
$\mathcal L_{\mathrm{row}}$ is not.
\item The implementation's child coefficients
$\omega_{ij} = \alpha\, n_{ij}/N_i$ with $\alpha = M_s/P_s$ satisfy
$\frac{1}{M_s}\sum_{i,j}\omega_{ij}\ell_{ij} = \mathcal L_{\mathrm{parent}}$,
Eq.~\eqref{eq:app-normalization-equivalence}.
\end{enumerate}
\end{proposition}
\begin{proof}
(1) Sum the coefficients of the children of parent $i$ in each objective.

(2) Expand
$\ell_{ij} = \frac{1}{n_{ij}}\sum_k m_{ijk}\,\bigl(-\log \pi_\theta(y_{ijk}\mid \widetilde h_{ij,<k})\bigr)$
and multiply by the parent coefficients from (1).

(3) By (2),
$\mathcal L_{\mathrm{parent}} = \frac{1}{P_s}\sum_i \frac{1}{N_i}\sum_{u \in i}\bigl(-\log \pi_\theta(y_u \mid h_u)\bigr)$
depends only on the occurrence multiset and on $N_i$, neither of which
changes under regrouping; the row weight $1/(M_s n_{ij})$ depends on the
partition through $M_s$ and $n_{ij}$.

(4) Substitute $\omega_{ij}$:
$\frac{1}{M_s}\sum_{i,j}\frac{M_s}{P_s}\frac{n_{ij}}{N_i}\ell_{ij}
= \frac{1}{P_s}\sum_i\sum_j \frac{n_{ij}}{N_i}\ell_{ij}$.
\end{proof}

\begin{remark}
Item (1) is why long multi-turn interactions, which render into more
children, dominate $\mathcal L_{\mathrm{row}}$; the parent-normalization
row of Table~\ref{tab:controlled-analysis} measures the effect of
removing this reweighting on identical renderings and masks. Items (3)
and (4) are what make the objective independent of native rendering
details. Finite-run updates can still depend on sample order and
batching, as noted in Appendix~\ref{app:parent-normalization}.
\end{remark}

\section{Additional Experimental Analysis}
\label{app:additional-results}

The main text reports the five-scale paired study, representative public
agent-training baselines, and the matched 8B component controls.
This appendix complements those results in two ways.
First, it reports the complete public-checkpoint reference set evaluated
under our unified benchmark protocol.
Second, it provides additional analysis of the matched controls and
structural-support allocation.

\subsection{Complete Public-Checkpoint Reference Set}
\label{app:public-checkpoints}

The main tables report representative 8B checkpoints from the public
agent families, together with MUA-RL-32B.
Table~\ref{tab:public-checkpoint-summary} expands this comparison to all
sixteen public checkpoints in the current quantitative reference set.
All models use the benchmark populations and scoring rules described in
Appendix~\ref{app:protocol}, while model-native prompting and tool-call
interfaces are preserved where necessary.
These results therefore provide a common evaluation reference rather than
a matched training comparison.

\begin{table}[H]
\centering
\caption{\textbf{Complete public-checkpoint evaluation} (\%).
Tau scores are domain macro averages; Mean equally averages Tau1, Tau2,
and BFCL. \method{}-8B is shown for reference.}
\label{tab:public-checkpoint-summary}
\small
\setlength{\tabcolsep}{5.2pt}
\renewcommand{\arraystretch}{1.06}
\begin{tabular*}{\linewidth}{@{\extracolsep{\fill}}lrrrr@{}}
\toprule
Model & Tau1 & Tau2 & BFCL & Mean \\
\midrule
\multicolumn{5}{@{}l}{\textit{MUA-RL}} \\
\quad 8B & 34.65 & 25.83 & 58.17 & 39.55 \\
\quad 14B & 47.57 & 45.00 & 64.05 & 52.21 \\
\quad 32B & 54.04 & 45.00 & 72.48 & 57.17 \\
\midrule
\multicolumn{5}{@{}l}{\textit{Simia (Qwen3-8B)}} \\
\quad Tau-SFT & 36.61 & 42.50 & 43.08 & 40.73 \\
\quad Tau-SFT-RL & 39.17 & 41.67 & 42.50 & 41.11 \\
\quad AgentBench-SFT & 18.35 & 16.67 & 64.37 & 33.13 \\
\quad OfficeBench-SFT\textsuperscript{*}
& 19.04 & 9.17 & 65.82 & 31.34 \\
\midrule
\multicolumn{5}{@{}l}{\textit{Simia (Qwen2.5)}\textsuperscript{$\dagger$}} \\
\quad Tau-SFT (7B) & 39.74 & 35.00 & 39.73 & 38.16 \\
\quad Tau-SFT-RL (7B) & 32.87 & 33.33 & 41.03 & 35.75 \\
\quad Tau-SFT (32B) & 48.39 & 45.83 & 43.43 & 45.88 \\
\quad AgentBench-SFT (7B) & 12.04 & 14.17 & 42.25 & 22.82 \\
\quad OfficeBench-SFT (7B) & 19.48 & 15.00 & 41.23 & 25.24 \\
\quad OfficeBench-SFT-RL (7B) & 15.30 & 17.50 & 42.16 & 24.99 \\
\midrule
\multicolumn{5}{@{}l}{\textit{Other public agents}} \\
\quad LoopTool-8B & 23.57 & 24.17 & 68.66 & 38.80 \\
\quad EnvFactory-8B (thinking) & 40.13 & 26.67 & 77.55 & 48.12 \\
\quad MidTool-RL-8B & 24.87 & 29.17 & 73.73 & 42.59 \\
\midrule
\rowcolor{T2GTableBlue}
\textbf{\method{} (Qwen3-8B)}
& 51.13 & 74.17 & 78.62 & \textbf{67.97} \\
\bottomrule
\end{tabular*}

\smallskip
\begin{minipage}{\linewidth}
\footnotesize
\textsuperscript{*} The OfficeBench Qwen3-8B BFCL diagnostic uses modified
reasoning-token handling; its BFCL and derived Mean are excluded from the
main comparison. \textsuperscript{$\dagger$} Qwen2.5 uses its native
assistant prefix.
\end{minipage}
\end{table}

\subsection{Cross-Benchmark Behavior}

The complete checkpoint set reinforces the distinction between
function-calling accuracy and stateful task execution.
The Qwen3 Simia checkpoints trained on Tau trajectories transfer more
strongly to Tau than their AgentBench- or OfficeBench-trained counterparts.
LoopTool, EnvFactory, and MidTool obtain comparatively strong BFCL scores
without corresponding Tau2 performance.
These differences motivate reporting all three benchmarks rather than a
single function-calling score.

The largest separation in this reference set occurs on Tau2, where
\method{}-8B exceeds every re-evaluated public checkpoint, including the
larger MUA-RL variants.
Across Tau1, Tau2, and BFCL, it also attains the highest equal-weight mean
among the public checkpoints evaluated here.

\subsection{Matched Component Controls}
\label{app:controls}

The main-text ablation in Table~\ref{tab:ablation} reports four matched
Qwen3-8B comparisons.
Because the interventions operate on different stages of the construction
procedure, each comparison has its own matched reference; the resulting
gains are therefore not an additive decomposition of the full-method
improvement.

\paragraph{Targeted expansion.}
The reference uses the existing certified interaction corpus, while the
proposed condition adds experience synthesized around under-supported
structures before applying the same downstream selection and training
procedure.
This comparison changes candidate supply and therefore does not hold the
training-token budget fixed.

\paragraph{Decision grounding.}
The semantic interactions and policy-visible histories are held fixed,
while the set of supervised assistant decisions changes.
The comparison therefore isolates whether success-derived targets are
filtered by decision-time evidence.

\paragraph{Structural selection.}
Both conditions draw from the same eligible interaction corpus and retain
the same task and capability anchors, 14,819 semantic interactions, and
target-token budget.
Structural allocation covers 6,360 canonical paths, compared with 5,418
under random surplus selection.

\paragraph{Parent normalization.}
The selected semantic interactions, rendered examples, target masks, and
sample order are held fixed.
Only the aggregate optimization mass assigned to each semantic interaction
changes.

\subsection{Structural Coverage under a Matched Budget}
\label{app:structural-response}

The structural-selection control connects experience allocation with
structural coverage.
At the same budget of 14,819 semantic interactions and matched target-token
mass, structural selection covers 6,360 canonical paths, compared with
5,418 under random surplus selection: an increase of 942 paths.

The corresponding performance change is concentrated on the stateful Tau
benchmarks.
As reported in Table~\ref{tab:ablation}, structural selection increases
Tau1 and Tau2 while slightly decreasing BFCL, yielding a 2.32-point gain
in the equal-weight three-benchmark mean.
Together, these results suggest that the allocation rule primarily changes
which stateful interaction structures are represented, rather than simply
increasing the amount of supervision.

\subsection{Support-Depth Sensitivity}
\label{app:support-analysis}

Table~\ref{tab:support-depth} varies the structural-support depth at fixed
Qwen3-8B.
Changing depth alters both the number of retained interactions and the
amount of supervised target mass, so this experiment should be interpreted
as sensitivity to the resulting training allocation rather than as an
isolated causal effect of repeated path exposure.

The three-benchmark mean is highest at depth two and decreases as
additional support is allocated.
Tau2 accounts for most of this decline, while BFCL remains comparatively
stable.
The result reinforces the role of support depth as an allocation parameter:
more repeated support is not automatically better.
The support schedule used in the scaling study was specified independently
of these sensitivity outcomes.

\FloatBarrier
\section{Qualitative Evidence}
\label{app:qualitative}

The examples illustrate behavior rather than estimate its frequency or
causally attribute improvements to individual components. They retain both
successful interactions and a remaining failure.

Appendix~\ref{app:baseline-case-comparison} walks through the same-task
comparison of Figure~\ref{fig:trajectory2gym-baseline-comparison},
Appendix~\ref{app:native-interaction-cases} shows how grounded
supervision treats a recorded training interaction
(Figure~\ref{fig:trajectory2gym-grounded-recovery}), and
Appendix~\ref{app:case-studies} shows six further paired cases against
re-evaluated checkpoints, including a remaining failure.

\subsection{Full Same-Task Comparison}
\label{app:baseline-case-comparison}

Figure~\ref{fig:trajectory2gym-baseline-comparison} compares three agents on
a Retail workflow requiring updates to a laptop order and a watch order.
The requested shipping address appears in an earlier order observation,
while the user profile contains a different address.
Completing the request requires both reusing that evidence and updating the
shipping address before an item modification changes the order status.

\stitle{Simia Tau-SFT.}
After observing the requested address, Simia repeatedly queries the user
profile, which returns the other address. It transfers the request to a
human without completing either update.

\stitle{MUA-RL-8B.}
The model first modifies the laptop item in order \texttt{\#W3730488}.
The returned status is \texttt{pending (item modified)}. Its subsequent
New York address update is rejected with \texttt{non-pending order cannot
be modified}; a later cancellation attempt is rejected as well. Changing
the user profile to New York does not change this order: subsequent order
reads still return Houston. The native task reward is 0.

\stitle{\method{}.}
The agent obtains confirmation, applies the retrieved address while the
order remains pending, and then updates the laptop and watch items.
The native task succeeds, whereas both external-baseline trajectories fail.

These independently generated dialogues share a task and system policy,
but not identical observation histories. The selected case illustrates
evidence use and prerequisite ordering, rather than estimating failure
frequency or proving a causal component effect.

\subsection{Decision-Level Grounding Example}
\label{app:native-interaction-cases}

Table~\ref{tab:app-retail-recovery-mask} shows the successful Retail
training interaction of Figure~\ref{fig:trajectory2gym-grounded-recovery},
the scenario that opens Section~\ref{sec:introduction}, which contains
one failed lookup followed by grounded recovery and several dependent
state-changing operations.
The failed call remains visible but receives no imitation loss; subsequent
decisions are supervised once their required evidence becomes visible.

\begin{figure}[htbp]
  \centering
  \includegraphics[width=\linewidth]{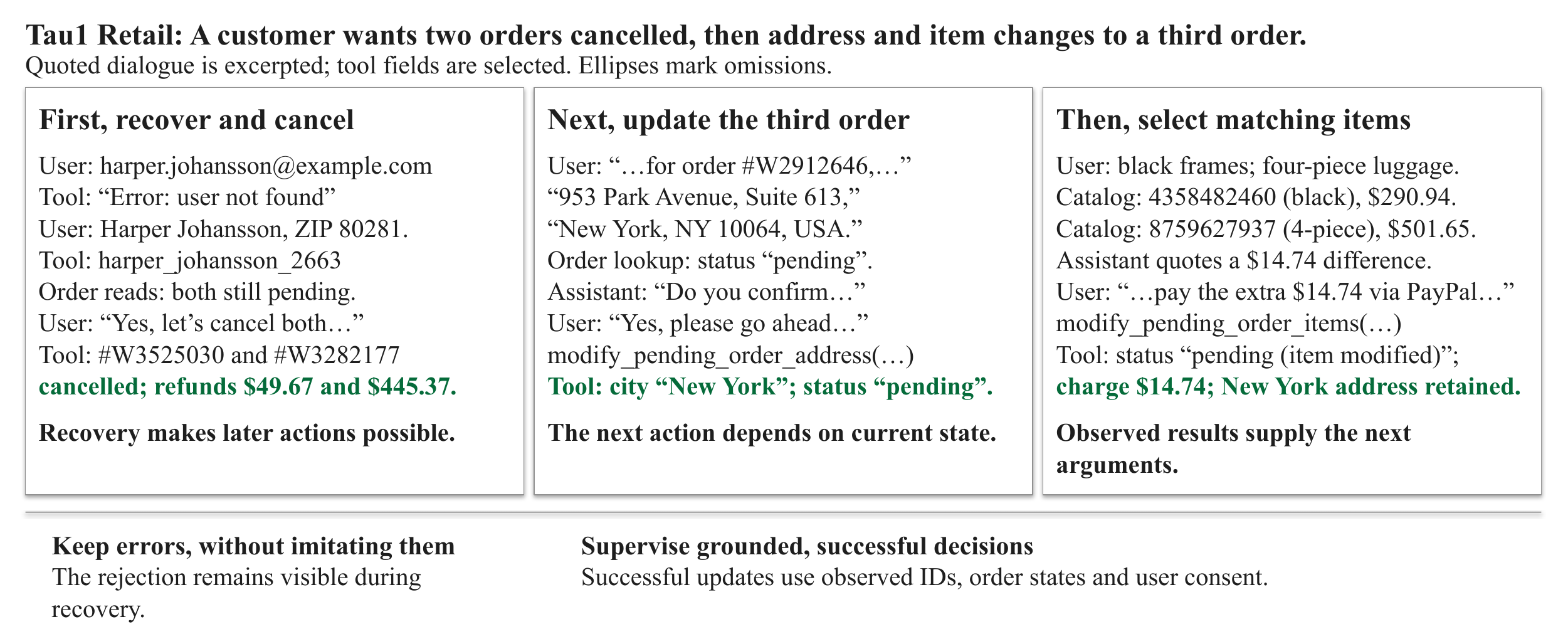}
  \caption{\textbf{Grounded supervision within a compositional interaction.}
  A failed email lookup remains visible context without receiving imitation
  loss. Subsequent visible evidence supports identity recovery, cancellation,
  address modification, and catalog-grounded item updates.
  Dialogue and tool fields are abridged for display.}
  \label{fig:trajectory2gym-grounded-recovery}
\end{figure}

\begin{table}[htbp]
\centering
\caption{Decision-level supervision in the multi-order Retail case.
The failed lookup uses an email supplied by the user: its argument is
visible, but its execution fails. The call and error remain context;
user messages and tool responses receive no imitation loss.}
\label{tab:app-retail-recovery-mask}
\small
\setlength{\tabcolsep}{4pt}
\begin{tabularx}{\linewidth}{@{}rP{0.25\linewidth}XP{0.16\linewidth}@{}}
\toprule
Decision & Assistant decision & Preceding visible evidence & Supervision \\
\midrule
0 & Ask for identity & User requests cancellations and updates & Supervised \\
1 & Lookup by email & User-supplied email; tool returns user not found & Context only \\
2--4 & Recover identity & Error, then user-supplied name and ZIP 80281 & Supervised \\
5--7 & Read two orders; confirm cancellations & User gives both order IDs; reads show pending status & Supervised \\
8--9 & Cancel both orders & User explicitly confirms both cancellations & Supervised \\
10--14 & Report results; confirm and change address & Third order is pending; user supplies and confirms the address & Supervised \\
15--17 & Read two catalogs; confirm item changes & Requested variants and PayPal ID from user; IDs and prices from catalogs & Supervised \\
18 & Modify both items & Catalog-derived replacements; user confirms the \$14.74 difference & Supervised \\
19 & Report updated order & Returned item IDs, address, status, and payment entry & Supervised \\
\bottomrule
\end{tabularx}
\end{table}

The agent recovers identity, confirms and cancels two pending orders,
then updates another order's address before changing its items.
Catalog observations supply the replacement identifiers, and user
confirmation authorizes the consequential changes.
The case illustrates how one visible interaction can contain both masked
failure and valid supervision for evidence-dependent recovery.

\FloatBarrier
\subsection{Additional Behavior Cases}
\label{app:case-studies}

The six cases in Figures~\ref{fig:case-intent-update}--\ref{fig:case-remaining-failure}
illustrate changing intent, feedback use, dependent operations,
state prerequisites, observation binding, and a remaining failure.
Each compares a named external agent with \method{} on the same task,
using their respective trained checkpoints. Dialogues evolve independently,
so these cases illustrate behavior rather than isolate the cause of an
aggregate performance difference. The figures excerpt dialogue and selected
tool fields, with omissions marked; $m$ is the original message index and
``log'' is the BFCL source-log index. Case-specific source paths, full
records, and native outcomes are retained in the accompanying JSON.

\begin{figure}[H]
  \centering
  \includegraphics[page=1,width=\linewidth,height=0.78\textheight,keepaspectratio]{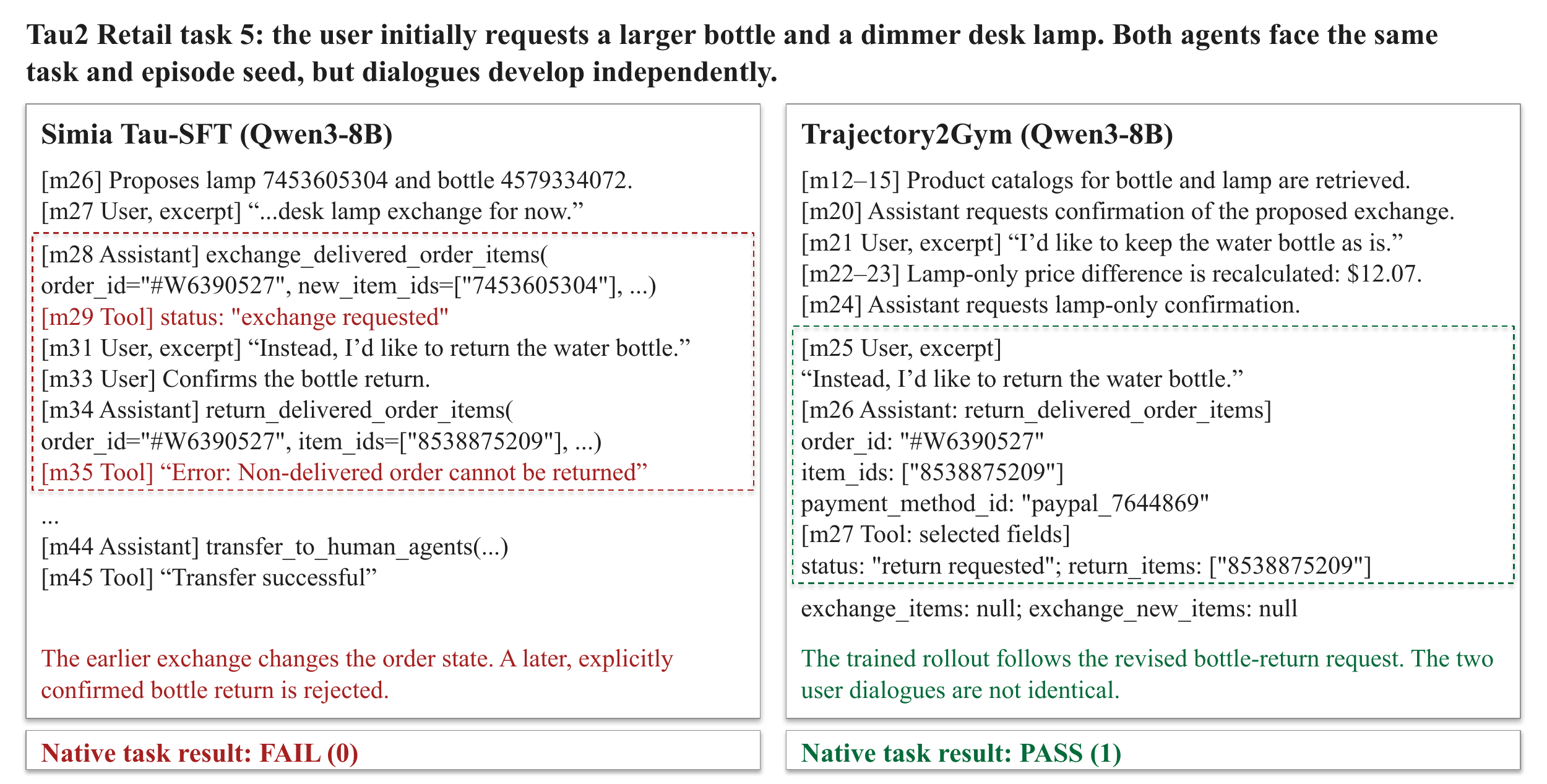}
  \caption{\textbf{Changing intent during a stateful interaction.}
  Simia submits an exchange before a later return request, which the changed
  order state prevents. In the independently generated \method{} dialogue,
  the user changes the request before submission and the return succeeds.
  The different histories make this a behavior illustration, not an
  identical-context comparison.}
  \label{fig:case-intent-update}
\end{figure}

\begin{figure}[H]
  \centering
  \includegraphics[page=2,width=\linewidth,height=0.78\textheight,keepaspectratio]{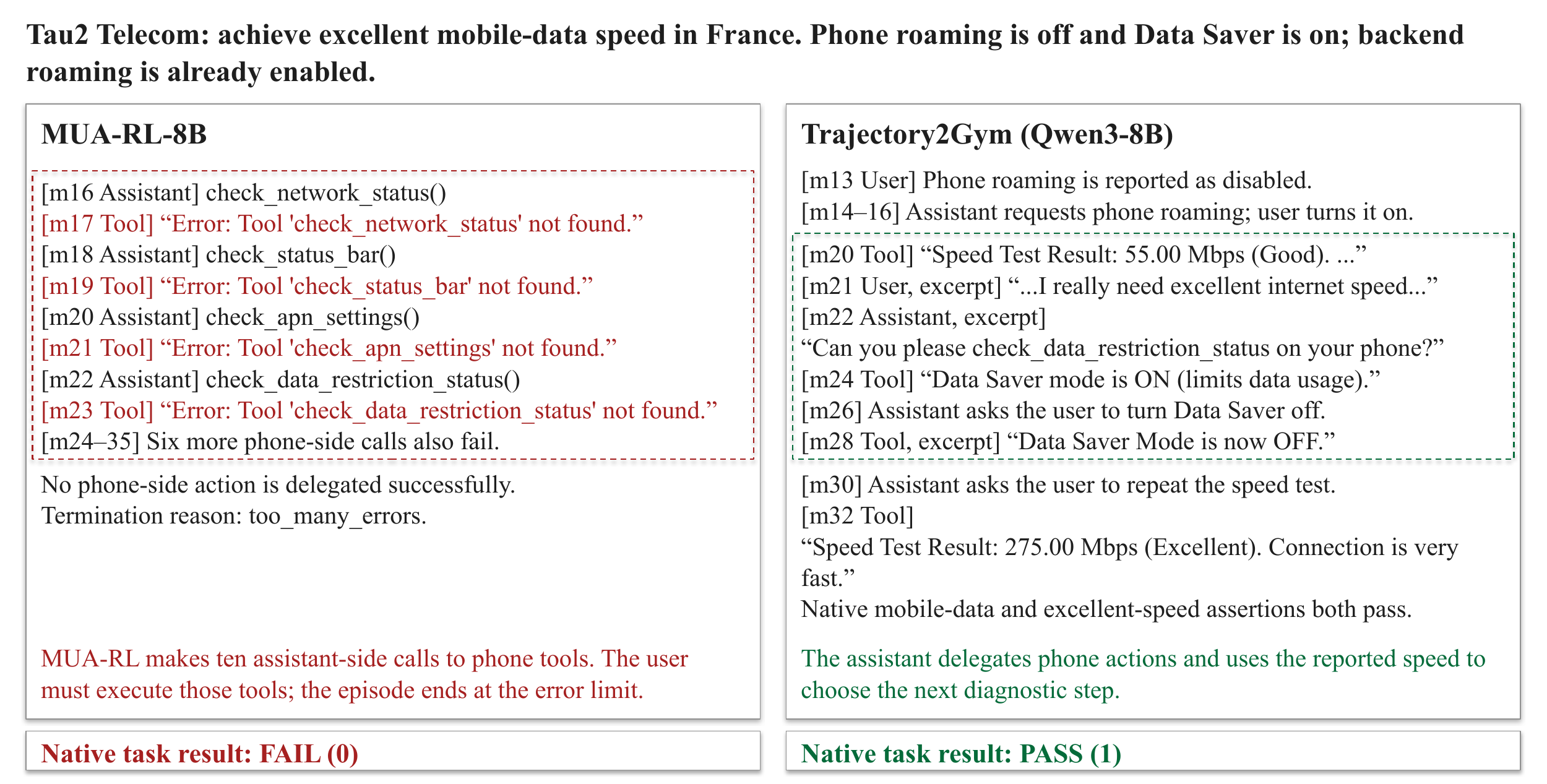}
  \caption{\textbf{Delegating phone-side actions and using feedback.}
  MUA-RL repeatedly attempts phone-side operations unavailable to the
  assistant. \method{} delegates those operations to the user and adapts
  its continuation to the reported device state and connection speed,
  completing the task.}
  \label{fig:case-feedback-recovery}
\end{figure}

\begin{figure}[H]
  \centering
  \includegraphics[page=3,width=\linewidth,height=0.78\textheight,keepaspectratio]{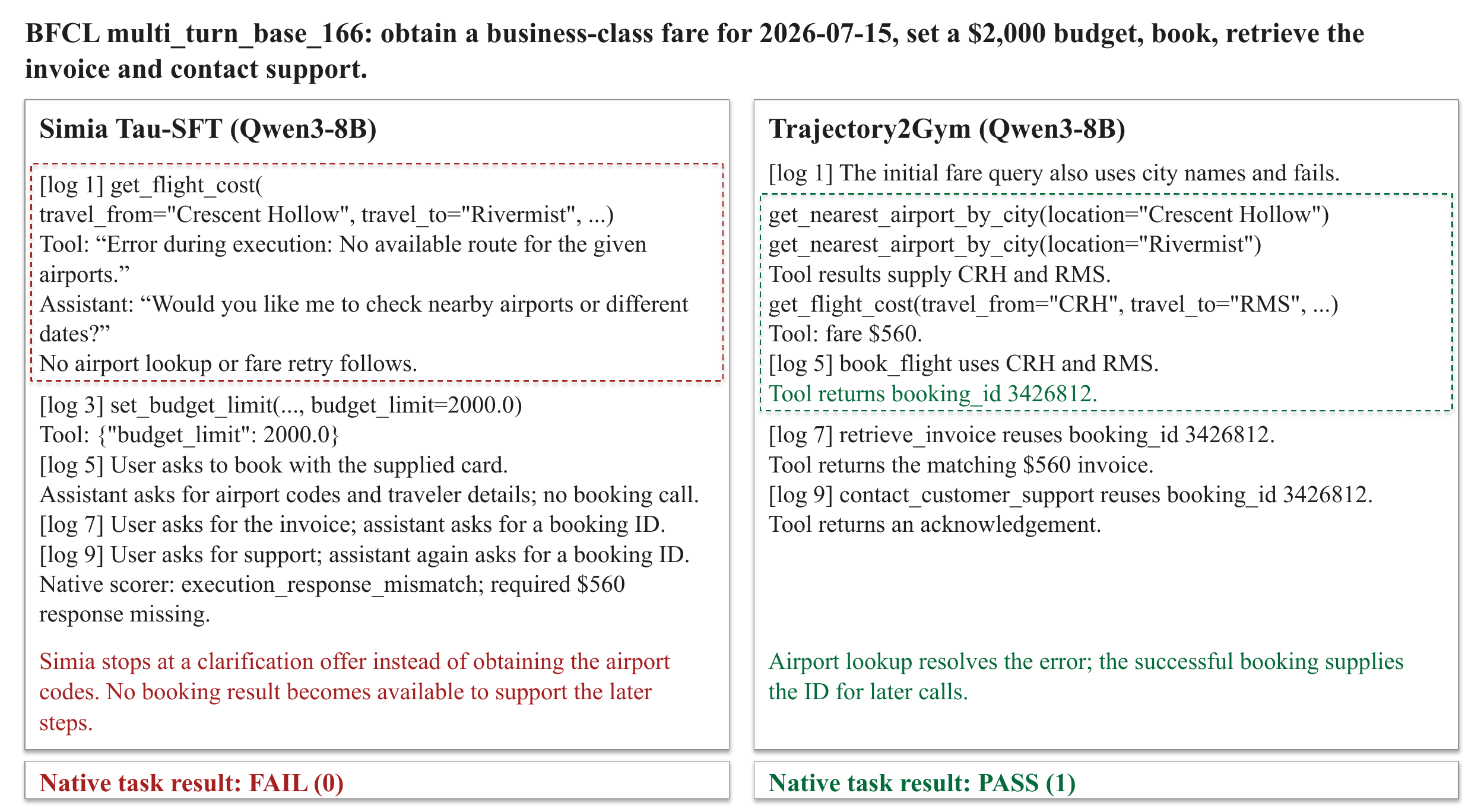}
  \caption{\textbf{Recovery enables dependent operations.}
  After a route error, Simia proposes a lookup but does not complete the
  booking. \method{} retrieves the needed airport codes, books the flight,
  and reuses the returned booking identifier for invoice and support calls.
  Only the \method{} interaction satisfies the native task checker.}
  \label{fig:case-dependent-tools}
\end{figure}

\begin{figure}[H]
  \centering
  \includegraphics[page=4,width=\linewidth,height=0.78\textheight,keepaspectratio]{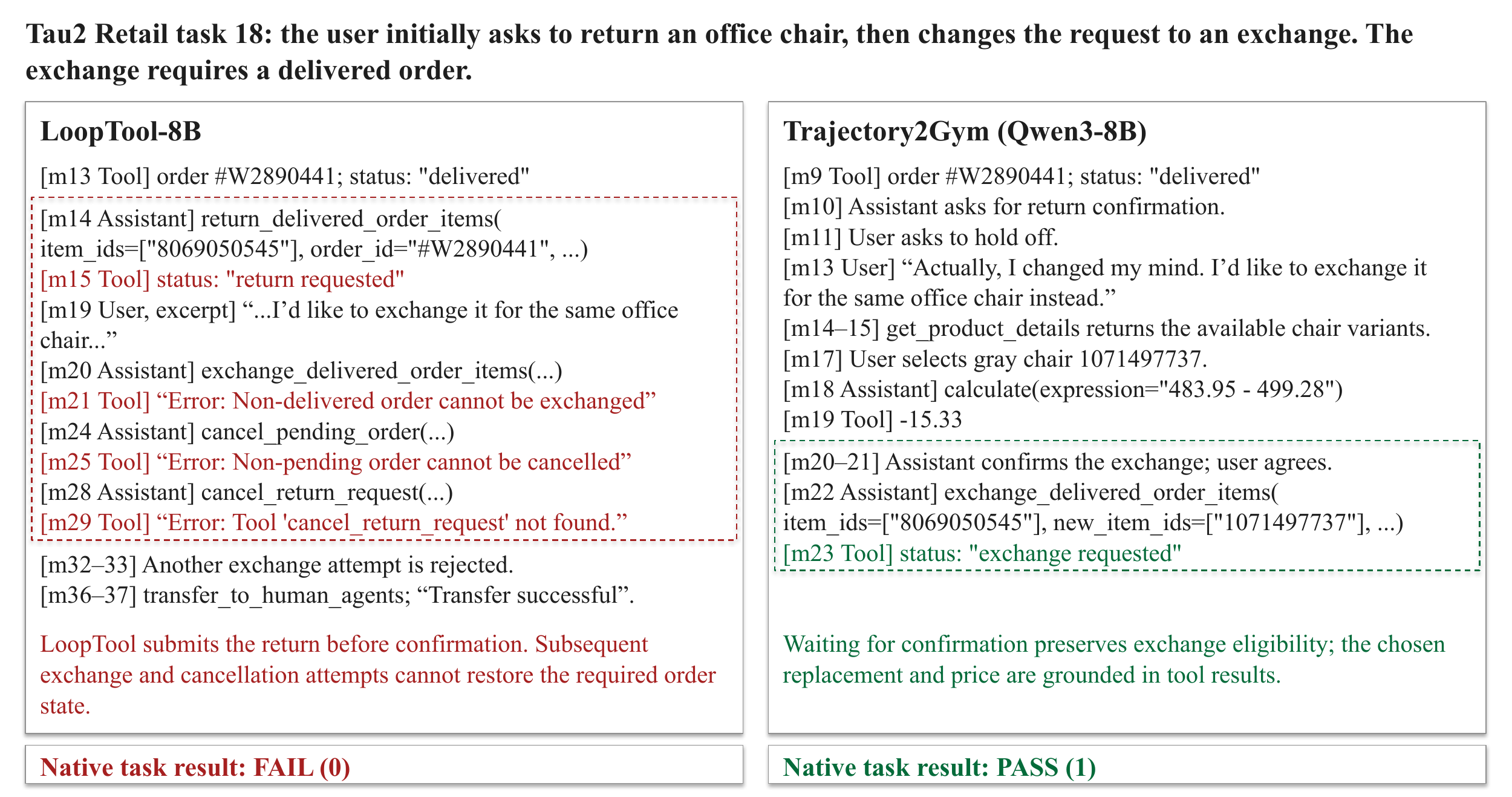}
  \caption{\textbf{An early return blocks later exchange attempts.}
  LoopTool submits a return before confirmation, leaving later exchange
  and cancellation attempts invalid. \method{} preserves the delivered
  state while obtaining confirmation, then submits the user-selected
  exchange.}
  \label{fig:case-state-prerequisite}
\end{figure}

\begin{figure}[H]
  \centering
  \includegraphics[page=5,width=\linewidth,height=0.78\textheight,keepaspectratio]{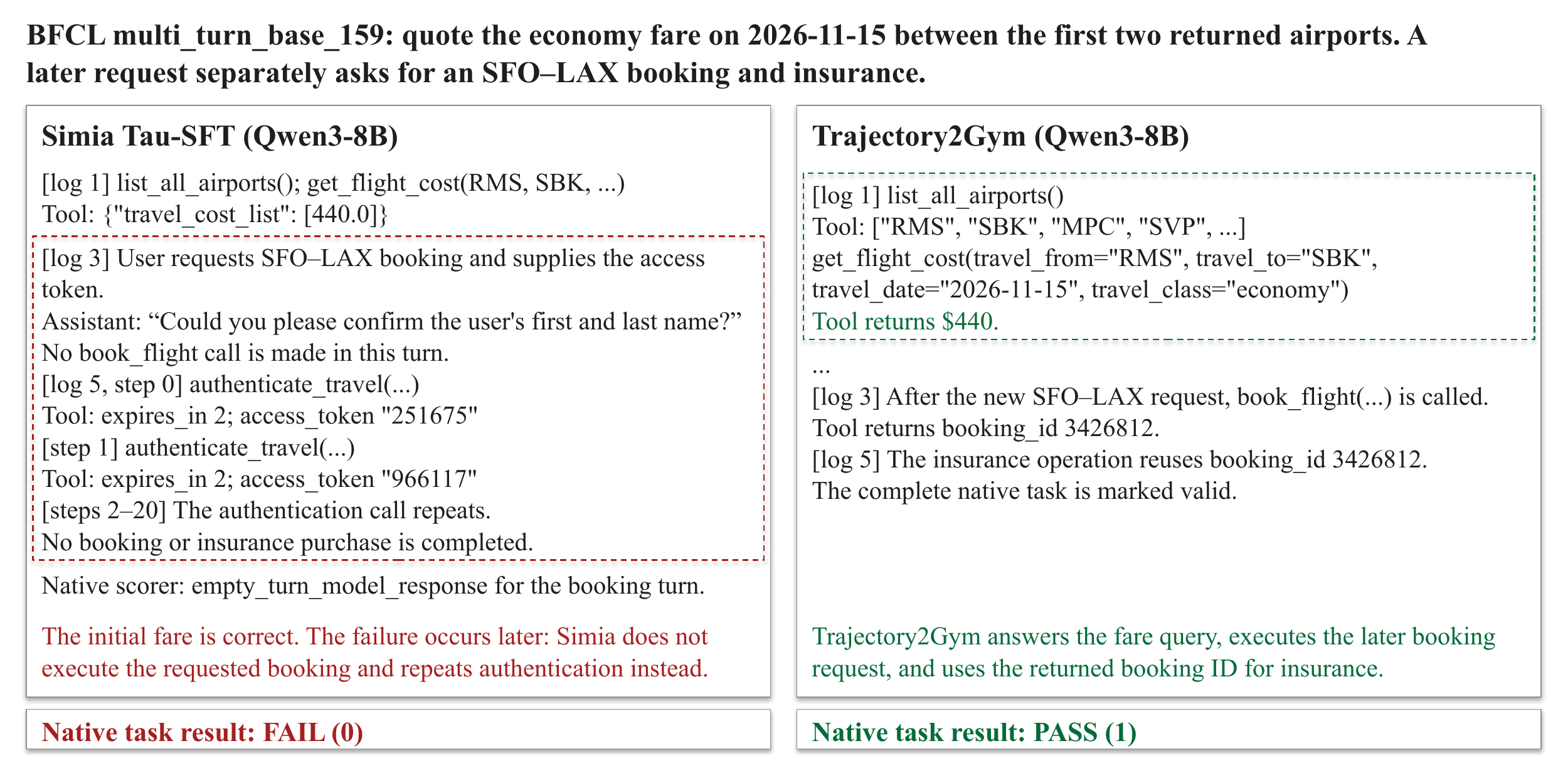}
  \caption{\textbf{A correct fare does not complete the workflow.}
  Both agents report the correct fare, but Simia does not execute the
  requested booking. \method{} books the flight and uses its returned
  identifier for the subsequent insurance operation, completing the
  dependent workflow.}
  \label{fig:case-observation-binding}
\end{figure}

\begin{figure}[H]
  \centering
  \includegraphics[page=6,width=\linewidth,height=0.78\textheight,keepaspectratio]{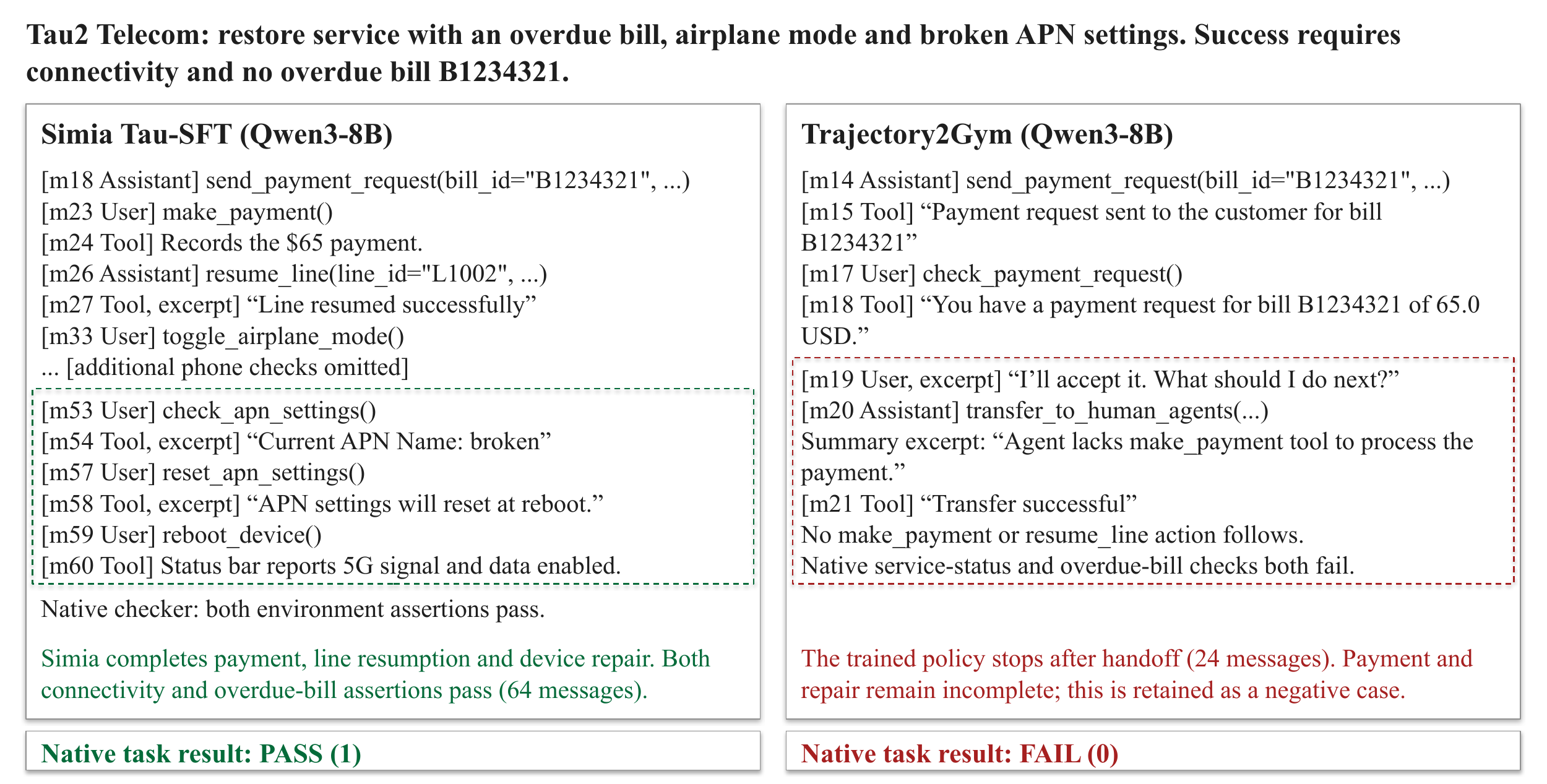}
  \caption{\textbf{Remaining failure against a successful external baseline.}
  Simia completes payment, line resumption, and device repair.
  \method{} transfers after sending a payment request, leaving the
  connectivity and overdue-bill checks unsatisfied.
  This case retains an external-agent success and a \method{} failure.}
  \label{fig:case-remaining-failure}
\end{figure}

\FloatBarrier

\end{document}